\documentclass{article}

\PassOptionsToPackage{numbers, compress}{natbib}

\usepackage[preprint]{neurips_2026}

\usepackage[utf8]{inputenc} 
\usepackage[T1]{fontenc}    
\usepackage{hyperref}       
\usepackage{url}            
\usepackage{booktabs}       
\usepackage{amsfonts}       
\usepackage{nicefrac}       
\usepackage{multirow}
\usepackage{microtype}      
\usepackage{xcolor}         
\usepackage{amsmath}
\usepackage{amsthm}
\usepackage{mathtools}
\usepackage{thmtools} 
\usepackage[linesnumbered,ruled,vlined]{algorithm2e}
\usepackage{cleveref}
\usepackage{graphicx}
\usepackage[most]{tcolorbox}

\newcommand{\std}[1]{\ensuremath{\text{\scriptsize{$\pm #1$}}}}
\definecolor{neighbor}{HTML}{009988}
\definecolor{entropy}{HTML}{EE7733}

\newtheorem{theorem}{Theorem}

\title{Context-weighted Discrete Flow Matching}

\author{
  Daniil Cherniavskii\thanks{Work done during an internship at Meta FAIR.} \\
  University of Amsterdam \\
  \texttt{d.cherniavskii@uva.nl}
  \And
  Daniel Severo \\
  Meta FAIR \\
  \texttt{dsevero@meta.com}
  \And
  Karen Ullrich \\
  Meta FAIR \\
  \texttt{karenu@meta.com}
}

\begin{document}

\maketitle

\begin{abstract}
  Discrete flow matching provides a flexible framework for generative modeling on discrete structures \citep{gat2024discrete}. However, the standard factorized training objective exposes the model to targets of varying difficulty, mixing well-conditioned, predictable tokens with ambiguous, high-entropy ones. We empirically demonstrate that the uncertainty over the value of each token is closely related to the density of available context in its neighborhood. Motivated by this observation, we propose a simple modification to the underlying Continuous-Time Markov Chain (CTMC) that incorporates local context information. Our context-weighted sampler improves generation quality with negligible computational overhead, while our scaled cross-entropy loss function reweights the training signal from different tokens and reduces generative perplexity by up to 63\% on OpenWebText \citep{Gokaslan2019OpenWeb}. Moreover, our approach matches a strong semi-autoregressive block diffusion baseline \citep{arriola2025block} in quality while retaining the ability for any-order generation. These results highlight the role of local context as an important factor in discrete generative modeling and show that simple context-aware modifications can significantly improve both sampling and training efficiency.
\end{abstract}

\section{Introduction}

\begin{figure}[ht]
    \centering
    \includegraphics[width=\linewidth]{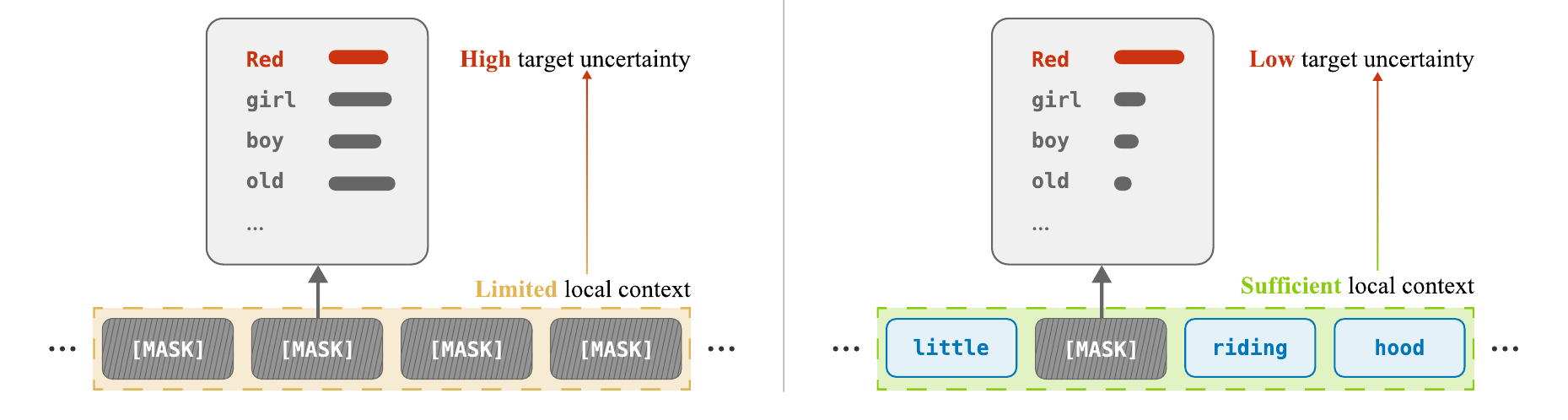}
    \caption{A motivational example. The target word ``Red'' is much more predictable in the second sentence because of the available local context.}
    \label{fig:example}
\end{figure}

Generative modeling over discrete data has recently expanded beyond autoregressive approaches to include diffusion and flow matching methods \citep{gat2024discrete, campbell2022continuous}. These methods have been successfully applied to image generation with discrete latents \citep{gu2022vector}, molecular design \citep{vignacdigress,kelviniuswyckoffdiff}, and text generation \citep{sahoo2024simple}. Among them, Discrete Flow Matching (DFM) is a particularly general instance of this family: it learns generative dynamics via continuous-time Markov processes \citep{lipman2024flowmatchingguidecode,gat2024discrete}, enabling parallel and any-order generation through learned probability paths.

In DFM and discrete diffusion, the probability path and factorized training objective decompose generation into coordinate-wise prediction problems. Unlike autoregressive models, which condition on a fixed prefix, these methods effectively marginalize over all possible token order permutations, requiring the model to predict tokens under a much broader set of partial contexts. As a result, tokens have widely varying amounts of conditioning information, leading to substantial variation in prediction difficulty \citep{kim2025train}. Some tokens are strongly constrained by context and easy to predict, while others remain ambiguous (see \cref{fig:example}). Existing approaches mitigate this imbalance by enforcing structured generation orders \citep{besnierhalton,arriola2025block,gat2025set} or adapting the sampler at inference time \citep{ben2025accelerated}, but they do not incorporate uncertainty directly into the generative dynamics.

In this work, we empirically observe that prediction difficulty is closely tied to available local context, measured by the number of unmasked tokens in each target token's neighborhood. Motivated by this, we modify the CTMC update probabilities to depend on local context while preserving the per-token marginal probability of being unmasked along the probability path. This keeps the same noise and data endpoints, but shifts intermediate updates toward better-conditioned tokens. The resulting formulation provides both an inference-time  sampler that requires no fine-tuning, and a scaled cross-entropy objective that reweights the training signal. Context-weighted sampling improves generation quality across domains: on OpenWebText \citep{Gokaslan2019OpenWeb}, it improves MAUVE by up to 24\%, while for molecule generation \citep{ramakrishnan2014quantum}, it increases the number of valid and novel samples by up to $2.8\times$ and $1.9\times$, respectively. Our loss objective reduces generative perplexity on OpenWebText by up to 63\%, and matches a strong block diffusion baseline while retaining any-order generation.

Our contributions are:
\begin{itemize}
    \item  We empirically show that prediction entropy is closely linked to available local context, with denser unmasked neighborhoods leading to lower uncertainty.
    \item  We modify the CTMC transition dynamics to depend on local context while preserving the per-token marginal unmasking probability.
    \item We derive two practical context-aware mechanisms: a sampler applicable to pre-trained models without fine-tuning, and a scaled cross-entropy objective that reweights the training signal from different tokens, resulting in significant gains in generation quality.
\end{itemize}

 The remainder of the paper is organized as follows: \cref{sec:background} reviews Discrete Flow Matching and probability path design, \cref{sec:method} presents our empirical motivation and context-weighted formulation, \cref{sec:experiments} evaluates the proposed sampler and training objective on text and molecular generation tasks, and \cref{sec:conclusion} summarizes the work and discusses limitations and future directions.
 
\section{Background}
\label{sec:background}

\subsection{Discrete Flow Matching}
\label{sec:dfm}

We consider finite discrete sequences $x\in\mathcal D^N$, where $\mathcal D$ is a finite vocabulary, and $N$ is the sequence length. Discrete Flow Matching (DFM)\citep{gat2024discrete} typically defines a conditional probability path $p_t(x\mid x_0,x_1)$ between a source noise sample $x_0\sim p_0$ and a data sample $x_1\sim p_{\mathrm{data}}$. In this work, we consider two common source distributions: a fully masked source, where all coordinates are initialized to a special mask token, and a uniform source, where coordinates are initialized from the vocabulary uniformly at random. For both sources, we say that coordinate \(i\) is unmasked whenever \(x^i=x_1^i\), and masked otherwise. 

The path satisfies the endpoint conditions
\[
p_0(x\mid x_0,x_1)=\delta(x_0, x),
\qquad
p_1(x\mid x_0,x_1)=\delta(x_1, x).
\]
where $\delta(z, x)$ denotes the Kronecker delta, equal to $1$ when $x=z$ and $0$ otherwise. This path is realized as a Continuous-Time Markov Chain (CTMC). For a small step $h$, its transition kernel can be written as
\[
\mathbb P(X_{t+h}=y\mid X_t=x)
=
\delta(x, y)+h\,u_t(y,x\mid x_0,x_1)+o(h),
\]
where $u_t(y,x\mid x_0,x_1)$ is the probability velocity, or transition rate, from state $x$ to state $y$, and $o(h)$ denotes a remainder term $r(h)$ satisfying $r(h)/h \to 0$ as $h\to 0$. DFM uses factorized velocities, where transitions change one coordinate at a time:
\[
u_t(y,x\mid x_0,x_1)
=
\sum_{i=1}^N
\delta(y^{\overline i},x^{\overline i})
u_t^i(y^i,x\mid x_0,x_1).
\]
where $i\in\{1,\ldots,N\}$ indexes a token position in the sequence, and $x^{\overline{i}}$ denotes $(x^1, \dots, x^{i-1},x^{i+1}, \dots, x^N)$, i.e. all coordinates of $x$ except $i$. This factorization enables efficient coordinate-wise simulation, although the path distribution $p_t(x\mid x_0,x_1)$ itself may not factorize across coordinates. The probability path and the velocity are linked by the Kolmogorov forward equation: choosing a path determines a family of corresponding velocities, and simulating the CTMC with this velocity recovers the path marginals. 

A common choice is the convex mixture path \citep{lipman2024flowmatchingguidecode,gat2024discrete}, defined independently for each coordinate as
\[
p_t(x^i\mid x_0,x_1)
=
(1-\kappa_t)\delta(x^i,x_0^i)
+
\kappa_t\delta(x^i,x_1^i),
\]
where $\kappa_t \in [0,1]$ is a scalar noise scheduler that increases with time. Its corresponding coordinate-wise velocity is
\[
u_t^i(y^i,x\mid x_0,x_1)
=
\frac{\dot\kappa_t}{1-\kappa_t}
\left[
\delta(y^i,x_1^i)-\delta(y^i,x^i)
\right].
\]
Thus each coordinate moves from its source value toward its data value at the same global rate, independent of how much context is available around that coordinate. A natural extension is to let the per-coordinate update rate depend on the current state $x_t$, so that coordinates can be updated differently while retaining coordinate-wise simulation.

The model is trained to predict the data endpoint from a partially noised state $X_t\sim p_t(\cdot\mid x_0,x_1)$ using the factorized Conditional Matching objective \citep{gat2024discrete}
\[
\mathcal L_{\mathrm{CM}}(\theta)
=
-\mathbb E_{t,X_0,X_1,X_t}
\sum_{i=1}^N
\log p_{1\mid t}^{\theta,i}(X_1^i\mid X_t).
\]
Thus training decomposes into coordinate-wise prediction problems, where each token is predicted from the same partially noised context but all coordinates are weighted uniformly.

\subsection{Probability Path Design}

The choice of probability path determines the family of conditional prediction problems encountered during training and sampling. In masked diffusion, different token orderings induce conditionals of widely varying difficulty, and learning all such factorizations can be substantially harder than following a favorable one \citep{kim2025train}. When the data admit an intrinsic structure, aligned orderings yield low-entropy, well-conditioned predictions, whereas unfavorable ones lead to ambiguous, high-entropy targets. Existing approaches address this by restricting the set of factorizations: for example, Halton schedules impose spatially uniform token selection in MaskGIT, while neighboring autoregressive decoding exploits local spatial dependencies in vision models \citep{besnierhalton,he2025neighboring}. Semi-autoregressive and block-wise decoding strategies similarly constrain the generation process by revealing continuous groups of tokens rather than arbitrary subsets \citep{arriola2025block,gat2025set,song2025seed,nie2025large,zhu2025llada,bie2025llada20scalingdiffusionlanguage,bie2026llada21speedingtextdiffusion,li2025autoregressive}. While these methods improve the induced conditionals, they hard-code an ordering, which may be suboptimal for individual samples and reduces the flexibility of order-agnostic generation. Test-time ordering methods instead adapt the generation order using confidence, margin, or entropy-based criteria \citep{zheng2024a,kim2025train,ben2025accelerated}, but rely on the model’s own uncertainty estimates and leave the training objective and probability path unchanged. General path formulations and kinetic-optimal path design offer a more direct way to modify discrete probability paths \citep{shaulflow}, but primarily target global path geometry rather than connecting locality of context to token-level prediction uncertainty. This leaves open how to incorporate such uncertainty into the probability path in a way that can inform both training and inference.


\section{Context-weighted Discrete Flow Matching}
\label{sec:method}

\subsection{Local Context and Token Uncertainty}
\label{sec:local_context}

\begin{figure}[ht]
    \centering
    \includegraphics[width=\linewidth]{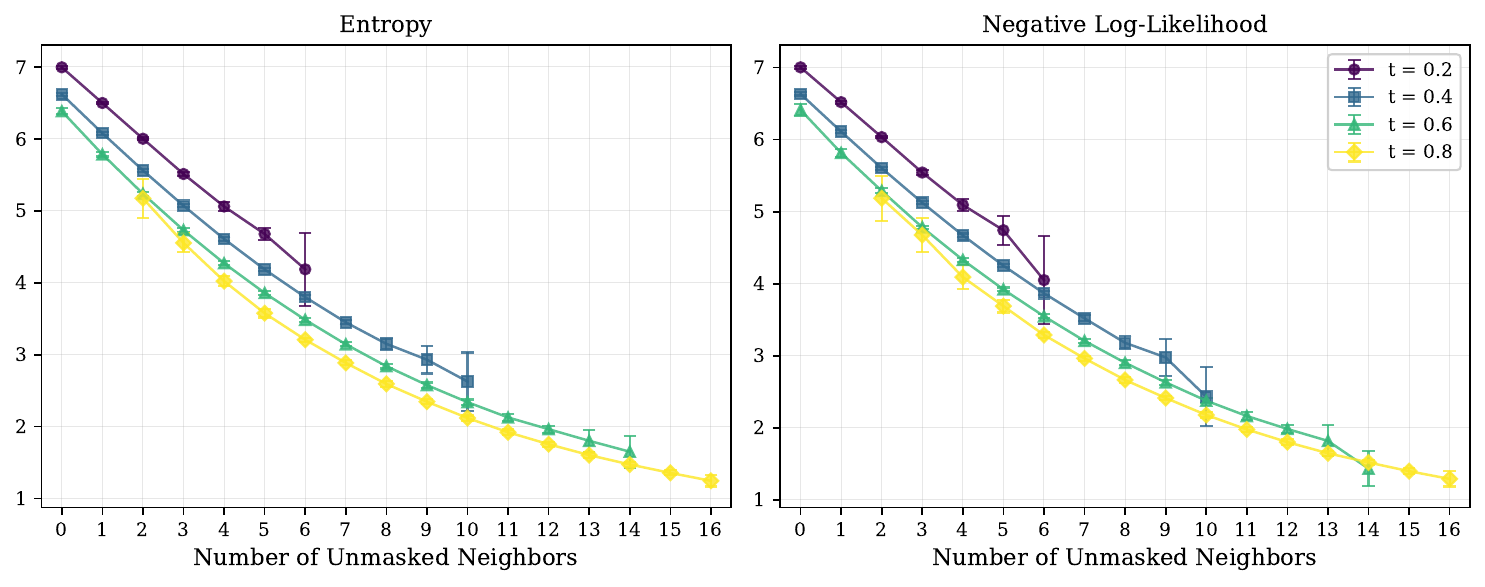}
    \caption{The average entropy (left) and negative log-likelihood (right) of the masked token vs. the number of its unmasked neighbors in the vicinity of size $2r = 16$. Both metrics decrease monotonically as the amount of available local context increases.}
    \label{fig:entropy_nll_boxplot}
\end{figure}

Following the view that probability path design induces the conditional prediction problems, we now ask what makes some tokens easier to predict than others. We empirically find that \textbf{local context provides a simple proxy for prediction difficulty} in Discrete Flow Matching: the number of unmasked neighbors around a token is strongly correlated with both the entropy of the model prediction and the negative log-likelihood of the true token.

Using a pre-trained DFM model on OWT~\citep{Gokaslan2019OpenWeb}, we sample partially masked states $X_t \sim p_t(\cdot \mid x_0, x_1)$ and for each masked token measure the prediction entropy and the negative log-likelihood as a function of the number of unmasked neighbors within a fixed local window. As shown in \cref{fig:entropy_nll_boxplot}, both quantities decrease consistently as more local context becomes available, indicating that nearby observed tokens make the prediction both less uncertain and more accurate (see more details in Appendix \cref{sec:add_results}).


This variation matters because the standard CM objective weights all coordinate-wise targets uniformly, even though their reducible learning signal may differ. For a fixed partially noised state $X_t$, with $X_1^i \sim p_{1\mid t}^i(\cdot\mid X_t)$, the expected loss of the model predictor $p_{1\mid t}^{\theta,i}$ decomposes as

\[
\mathbb{E}_{X_1^i\sim p_{1\mid t}^i(\cdot\mid X_t)}
\left[
-\log p_{1\mid t}^{\theta,i}(X_1^i\mid X_t)
\right]
=
\underbrace{
H\!\left(p_{1\mid t}^i(\cdot\mid X_t)\right)
}_{\text{irreducible uncertainty}}
+
\underbrace{
\mathrm{KL}\!\left(
p_{1\mid t}^i(\cdot\mid X_t)
\,\|\, 
p_{1\mid t}^{\theta,i}(\cdot\mid X_t)
\right)
}_{\text{model error}}.
\]
Thus the per-token loss contains both an irreducible uncertainty term and a reducible model error term. For high-entropy coordinates, a large part of the loss comes from intrinsic ambiguity in the conditional distribution, so a single sampled target gives a less informative training signal. Low-entropy coordinates, in contrast, correspond to more determined targets and provide sharper supervision. Ideally, a weighting scheme should emphasize the reducible component of the loss rather than uniformly weighting raw per-token losses, which can be dominated by irreducible uncertainty. Since the true entropy is unavailable during training, we use local context as an inexpensive proxy for this token-level difficulty.

\subsection{Context-weighted CTMC}

In this section, we propose the context-aware modification to the standard CTMC process \citep{gat2024discrete}. In order     to keep coordinate-wise simulation tractable, we will employ only factorized velocities. Consider a context weight function $\alpha \colon \mathcal{D}^N \to \mathbb{R}^N_+$. To integrate it into the CTMC, we define a new factorized velocity $v_t$ as:

$$
v_t^i(x^i, z \mid x_0, x_1) \coloneqq \alpha^i(z) \cdot u_t^i(x^i, z \mid x_0, x_1)
$$

Note that it is a valid probability velocity, since it satisfies the rate conditions:

$$
\sum_{x^i \in [d]} v_t^i(x^i, z)  =  \sum_{x^i \in [d]} \alpha^i(z) u_t^i(x^i, z) =  \alpha^i(z) \sum_{x^i \in [d]} u_t^i(x^i, z) = 0
$$

and $v_t^i(x^i, z) \geq 0, \forall x^i \neq z^i$ . We emphasize that $\alpha$ depends only on the $z$ (current state) and controls which token positions will be changed, but not the actual value of the token, since it is independent of $x$. This new velocity $v_t(y, x)$ generates a new conditional probability path $\tilde{p}_t(x \mid x_0, x_1)$, defined as the solution to the corresponding Kolmogorov forward equation (see \cref{sec:dfm} for details). A necessary requirement for $\tilde{p}_t$ is to connect noise and data distributions, imposing boundary conditions $\tilde{p}_{t=0}(x \mid x_0, x_1) = \delta(x, x_0)$ and $\tilde{p}_{t=1}(x \mid x_0, x_1) = \delta(x, x_1)$. Satisfying them in the general case requires solving the aforementioned equation for $\tilde{p}_t$, which we find intractable. However, for the standard convex mixture path, we are able to derive the conditions on the context weight function $\alpha^i$  to satisfy the boundary conditions and preserve the marginal distribution. 

\begin{theorem}
\label{thrm:main}
Let $\tilde p_t(\cdot \mid x_0,x_1)$ denote the context-weighted discrete path associated with the standard convex probability path. If
$$
\sum_{i=1}^N \alpha^i(x_t)\,\delta(x_t^i,x_0^i)
\;=\;
\sum_{i=1}^N \delta(x_t^i,x_0^i),
$$
then $\tilde p_t$ satisfies the endpoint constraints
$$
\tilde p_{0}(x\mid x_0,x_1)=\delta(x,x_0),
\qquad
\tilde p_{1}(x\mid x_0,x_1)=\delta(x,x_1),
$$
and induces the same distribution over the number of updated tokens as the standard convex path.


If, moreover, $\alpha$ is equivariant under circular shifts, in the sense that
\[
\alpha^{i+k}(\operatorname{shift}_k(x)) = \alpha^i(x)
\qquad \forall i,k,
\]
where indices are taken modulo $N$, and depends only on the binary mask reveal, i.e. 
\[
\alpha^i(x) = \bar{\alpha}^i(b(x))
\] 
where \(\bar{\alpha} \colon \{0,1\}^N \to \mathbb{R}^{N}_{+}\) and \(b^j(x) = \mathbf{1}[x^j = x_1^j]\), then for every coordinate $i$ the marginal law matches that of the standard convex path:
\[
\mathbb{P}(X_t^i=x_1^i \mid x_0,x_1)=\kappa_t.
\]
\end{theorem}

The first condition is equivalent to requiring that the weights average to \(1\) over the currently masked coordinates. Intuitively, this preserves the total update rate, and therefore the distribution of the number of tokens updated along the path. The second condition is a symmetry requirement: under circular-shift equivariant weighting, count preservation also yields matching per-coordinate marginals. The formal proof is given in \cref{sec:theory}.

\textbf{Inference-time sampling.} Our formulation enables a purely inference-time modification of the sampling procedure. Assuming \cref{thrm:main} conditions, i.e. the probability velocity $v_t$ remains factorized and the per-token marginals are preserved, we can alter the dynamics in-place, \textbf{without any fine-tuning or computational overhead.} Specifically, we inject a multiplicative weight $\alpha^i(x)$ into each coordinate of the Euler solver \citep{lipman2024flowmatchingguidecode,gat2024discrete} for $u_t$, obtaining a context-weighted update rule:

$$
\mathbb{P}(X_{t+h} = y \mid X_t = x) = \begin{cases}
    \exp\left(\boldsymbol{\alpha}(x) \cdot h u_t(x, x)\right), & y = x \\
    \frac{u_t(y, x)}{|u_t(y, x)|}\left(1 - \exp\left(\boldsymbol{\alpha}(x) \cdot h u_t(x, x)\right)\right), & y \neq x
\end{cases}
$$

Using a pre-trained Discrete Flow Matching model, at each timestep $t$ and for each token $x_t^i$, we weight the jump coefficients with $\alpha^i(X_t)$. This modification leaves the computational complexity almost unchanged, while allowing the sampling process to adapt to the local context.

We introduce two context-weighted solvers: {\color{neighbor}Neighbor-weighted} and {\color{entropy}Entropy-weighted}. Both define position-wise weights $\alpha^i(x_t)$, normalized with a softmax inverse temperature $s$ so that they satisfy the boundary conditions of \cref{thrm:main}. We provide a verification of this in the Appendix \cref{sec:alg_correct}.

For the {\color{neighbor}Neighbor-weighted} solver, we use the number of unmasked tokens in a local window of radius $r$:
\[
{\color{neighbor}\alpha}_{\mathrm{NW}}^i(x)
\propto
\sum_{j=i-r}^{i+r} \mathbf{1}[x^j = x_1^j]
\]

For the {\color{entropy}Entropy-weighted} solver, we use token-level predictive entropy available at test time:
\[
{\color{entropy}\alpha}_{\mathrm{EW}}^i(x)
\propto
H\left(p_{1\mid t}^{\theta,i}(\cdot \mid x)\right)^{-1}
\]

In practice, we use $s>0$ for Neighbor-weighted and $s<0$ for Entropy-weighted, so locally well-contextualized or low-entropy positions receive larger weights. Unlike \citet{ben2025accelerated}, our Entropy-weighted sampler is applicable to arbitrary source distributions and uses entropy only to select which positions to update, while preserving the stepwise distribution of the number of updated tokens.






\textbf{Train-time path $\tilde{p}_t$ sampling.} For general context weights $\alpha^i(x)$, the context-weighted path $\tilde{p}_t$ is not available in an analytic form, and therefore, simulation-free sampling is impossible. Moreover, the token values are no longer conditionally independent for $\tilde{p}_t$, so sampling cannot be reduced to a simple parallel factorized procedure. A straightforward approach would therefore rely on time-wise simulation of the path, which is quite costly. Instead, we leverage the fact that the distribution of the number of unmasked tokens $M_t$ is known, which allows us to construct an exact token-wise sampler with $O(N)$ complexity (see \cref{alg:path_sampling}). We first sample a number of unmasked tokens $M_t$, and then update the tokens one-by-one. This eliminates the need for costly simulation along the time axis while still returning exact samples from $\tilde{p}_t$. We prove the exactness of the algorithm in the Appendix \cref{sec:alg_correct}.

\begin{algorithm}[ht]
\caption{Train-time $\tilde{p}_t(x)$ sampling}
\label{alg:path_sampling}
\KwIn{$N$, $\kappa_t$, $x_0$, $x_1$, $\delta$, $\alpha$}
\KwOut{A sample $\tilde{X}_t \sim \tilde{p}_t$}
Sample $m \sim \mathrm{Binomial}(N, \kappa_t)$ \tcp*[r]{Total number of unmasked tokens}
Set $\tilde{X}_t \gets x_0$\;
\For{$\text{step} = 1$ \KwTo $m$}{
    Sample $i \sim (1-\delta(\tilde{X}_t^i, x_1^i)) \cdot \alpha^i(\tilde{X}_t)$ \tcp*[r]{Choose masked token position}
    Set $\tilde{X}_t^i \gets x_1^i$ \tcp*[r]{Unmask the token}
}
\Return $\tilde{X}_t$\;
\end{algorithm}

\textbf{Scaled Cross-Entropy.} Motivated by the empirical link between prediction uncertainty and the available local context (see \cref{sec:local_context}), we introduce \emph{Scaled Cross-Entropy} as a plug-in alternative to the standard Discrete Flow Matching objective:
$$
\mathcal{L}_{\alpha}(\theta)
=
-
\mathbb{E}_{t, X_1, X_0, X_t}
\left[
\sum_{i=1}^N \alpha^i(X_t)\,\log p_{1 \mid t}^{\theta, i}(X_1^i \mid X_t)
\right].
$$
Here, the context weight function $\alpha^i$ reweights the contribution of each coordinate, upweighting better-conditioned coordinates and downweighting ambiguous ones. It leaves the train-time path sampling unchanged, and therefore retains the efficiency of standard training and any-order sampling at inference time, while capturing the context-weighted importance of different token updates.


\section{Experiments}
\label{sec:experiments}


We evaluate our method on two domains: text and molecular generation. For text, we use OpenWebText (OWT) \citep{Gokaslan2019OpenWeb}, a large-scale corpus, while for molecules we use QM9 in SMILES format \citep{ramakrishnan2014quantum}, a significantly smaller dataset that allows us to study performance in the low-data regime. We follow the MDLM architecture and training setup to ensure fair comparison \citep{sahoo2024simple}. Our OWT model has approximately 170M parameters; for QM9, we modify only the vocabulary, resulting in a model with approximately 92M parameters. For text, we report generative perplexity (Gen. PPL), computed with GPT-2\citep{radford2019language}, entropy to measure text diversity, and MAUVE as a balanced quality–diversity metric \citep{pillutla2021mauve}. For QM9, we report the number of valid molecules and the number of novel samples (out of 1024 generations). All setup details are provided in the appendix \cref{sec:app_setup}.

\subsection{Inference-time context-weighted sampling}

\begin{figure}
    \centering
    \includegraphics[width=\linewidth]{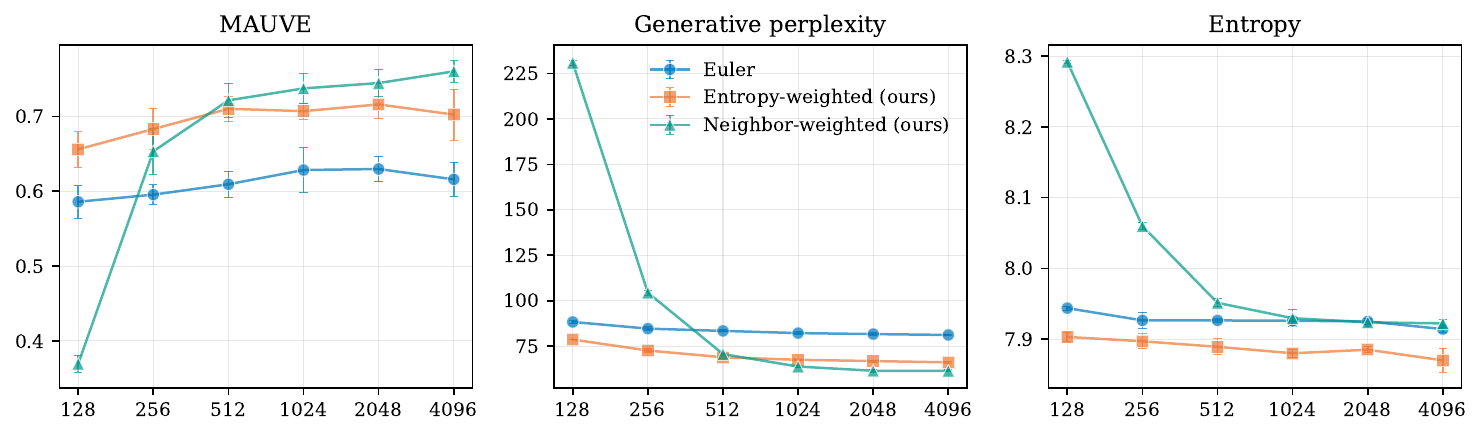}
    \caption{\textbf{Neighbor-weighted sampling improves generation quality without fine-tuning or computational overhead.} At larger NFE, it outperforms both the Euler baseline and the model-uncertainty-based Entropy-weighted solver in MAUVE and generative perplexity, while maintaining comparable token-level entropy.}
    \label{fig:nfe_vs_quality_uniform}
\end{figure}

\begin{figure}
    \centering
    \includegraphics[width=0.8\linewidth]{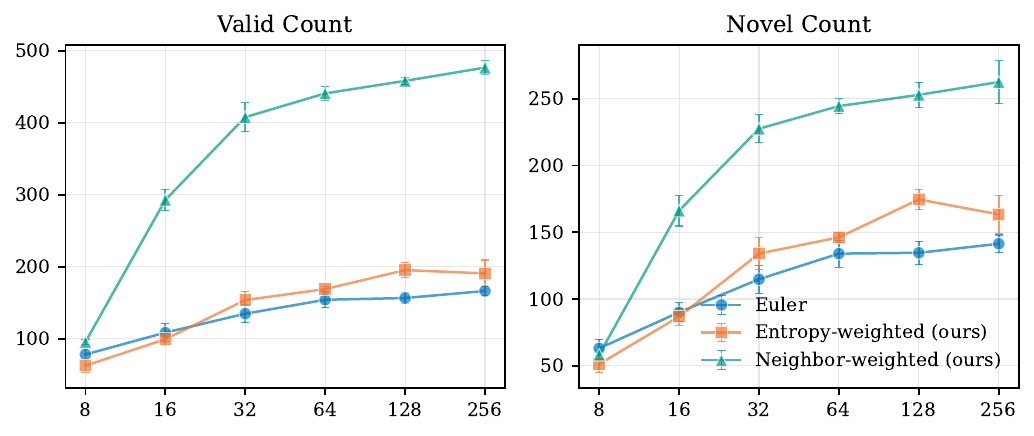}
    \caption{\textbf{Local context can be more reliable than model uncertainty in low-data regimes.} On the small QM9 dataset, Neighbor-weighted sampling outperforms both Euler and Entropy-weighted baselines across all NFE levels, nearly tripling the number of valid molecules and doubling the number of novel molecules.}
    \label{fig:nfe_vs_quality_qm9_mask}
\end{figure}



\textbf{Context-weighted sampling improves generation quality at negligible computational cost.} On OpenWebText with uniform source noise, the Neighbor-weighted solver improves MAUVE by up to $23\%$ and reduces generative perplexity by $22\%$ relative to Euler, while preserving comparable entropy (see \cref{fig:nfe_vs_quality_uniform}). Similar perplexity gains are observed with masked-source noise (see \cref{fig:nfe_vs_quality_mask} in the Appendix).

The two solvers behave differently across compute regimes. Entropy-weighted sampling is stronger at low NFE, whereas Neighbor-weighted sampling benefits from more update steps: it overtakes Euler in MAUVE from NFE $=256$ and in perplexity from NFE $=512$, and eventually surpasses Entropy-weighted sampling from NFE $=1024$. We attribute the weaker low-NFE behavior to local inconsistencies caused by neighboring tokens being updated too independently when only few updates are performed.


\textbf{Local context information is robust in the low-data regime.} On QM9, Neighbor-weighted sampling yields much larger gains than Entropy-weighted sampling, increasing valid molecules by approximately $\times 2.8$ and novel molecules by approximately $\times 1.9$ over Euler (see \cref{fig:nfe_vs_quality_qm9_mask}). This suggests that, when data is limited, model-based uncertainty estimates can be noisy, while local context remains a stable signal for guiding updates.



\begin{figure}[ht]
    \centering
    \includegraphics[width=\linewidth]{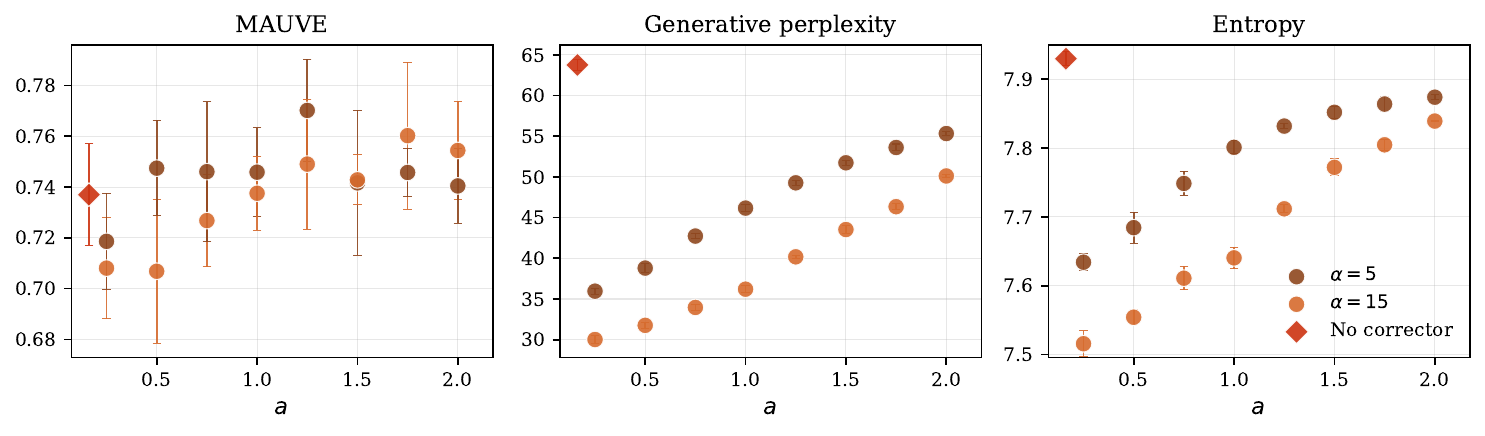}
    \caption{\textbf{Context-weighted sampling is complementary to predictor--corrector sampling.} Adding Neighbor-weighted updates to a predictor--corrector sampler further lowers generative perplexity and slightly improves MAUVE, showing that local-context weighting provides an additional quality-improving mechanism.}
    \label{fig:corrector_uniform_nw}
\end{figure}

\textbf{Context-weighted sampling is complementary to predictor--corrector schemes.}
To test whether local-context weighting provides gains beyond stronger samplers, we combine Neighbor-weighted scaling with the predictor--corrector framework of \citet{gat2024discrete}. We search over correctors of the form
\[
\alpha_t = 1 + \alpha t^a(1-t)^b,
\]
with $a=b$ to reduce the search space. On a pre-trained DFM with uniform source noise, the combined sampler further improves MAUVE and generative perplexity, at the cost of lower entropy (see \cref{fig:corrector_uniform_nw}).


\subsection{Train-time modifications and ablations}

\begin{table}[ht!]
    \caption{Context-weighted Discrete Flow Matching approaches strong semi-autoregressive block diffusion baselines in MAUVE while substantially improving over non-autoregressive baselines.} 
    \label{tab:baselines}
    \centering
    \begin{tabular}{ll|ccc}
        \toprule
        Type & Method & MAUVE ($\uparrow$) & Gen. PPL ($\downarrow$) & Entropy ($\uparrow$ )\\
        \midrule
        Data & -  & 0.890\std{0.008} & 14.66\std{0.13} & 7.86\std{0.01} \\
        \midrule
        (Semi-)AR& AR  & 0.817\std{0.017} & \textbf{12.72}\std{0.25} & 7.57\std{0.02} \\
        & BD3LM ($L'=4$)  & 0.784 \std{0.025} & 24.19\std{0.08} & 7.69\std{0.01}\\
        & BD3LM ($L'=8$)  & 0.734\std{0.022} & 29.47\std{0.08}& 7.70\std{0.01} \\
        & BD3LM ($L'=16$)  & 0.727\std{0.023} & 32.95\std{0.39} & 7.71\std{0.01} \\
        \addlinespace 
        Non-AR & SEDD & 0.574\std{0.024} & 110.23\std{0.63} & 8.14\std{0.01}\\ 
        & MDLM  & 0.685\std{0.021} & 41.73\std{0.49} & 7.73\std{0.01}\\ 
        & DFM, Uniform & 0.628\std{0.030} & 82.17\std{0.84} & 7.93\std{0.01}\\
        & DFM, Mask & 0.672\std{0.024} & 40.90\std{0.08} & 7.64\std{0.01} \\
        \midrule
        & DFM+SCE, Uniform & 0.777\std{0.017} & 30.20\std{0.22} & 7.66\std{0.01} \\
        & DFM+SCE, Mask & 0.690\std{0.010} & 38.85\std{0.50} & 7.67\std{0.01} \\
        
        \bottomrule
    \end{tabular}
\end{table}

\textbf{Scaled Cross-Entropy brings Discrete Flow Matching close to strong semi-autoregressive baselines with a minimal modification.} Scaled Cross-Entropy yields large improvements in generation quality, despite requiring only a simple plug-in change to the standard loss (see \cref{tab:baselines}). In the uniform-source setting, it improves MAUVE ($0.628 \to 0.777$) and \textbf{reduces generative perplexity by 63\%} ($82.17 \to 30.20$). This is sufficient to bring Discrete Flow Matching close to the strong semi-autoregressive block diffusion baselines ($L'=4,8$), \textbf{while preserving the flexibility of any-order generation} and the standard Euler sampler. Moreover, it clearly outperforms other non-autoregressive baselines such as SEDD and MDLM, showing that better weighting of the training signal alone can account for a substantial portion of the quality gap.



%



\begin{table}[ht!]
    \caption{DFM trained with scaled cross-entropy achieves lower perplexity while maintaining reasonable entropy.} 
    \label{tab:loss_functions}
    \centering
    \begin{tabular}{ll|ccc|cc}
        \toprule
         \multicolumn{2}{c|}{} & \multicolumn{3}{c|}{OWT} & \multicolumn{2}{c}{QM9} \\
         \cmidrule(lr){3-5}\cmidrule(lr){6-7}
         Source & Loss & MAUVE $(\uparrow)$ & Gen.\ PPL $(\downarrow)$ & Entropy $(\uparrow)$ & Valid $(\uparrow)$ & Novel $(\uparrow)$ \\
         \midrule
         \multirow{4}{*}{Uniform} & CE & 0.628 \std{0.030} & 82.17 \std{0.84} & 7.93\std{0.01} & 475.4\std{17.1} & 287.0\std{15.6} \\
          & NELBO & 0.677\std{0.035} & 74.18\std{0.40} & 7.88\std{0.01} & 481.4\std{14.4} & 299.0\std{15.8}\\
          & Bregman & 0.682\std{0.015} & 63.06\std{0.53} & 7.85\std{0.01} & 494.4\std{10.4} & \textbf{308.0}\std{8.7} \\
          & \textbf{SCE (ours)} & \textbf{0.777}\std{0.017} & \textbf{30.20}\std{0.22} & 7.66\std{0.01} & \textbf{556.0}\std{16.1} & 297.6\std{10.1} \\
         \midrule
         \multirow{4}{*}{Mask} & CE & 0.672\std{0.024} & 40.90\std{0.31} & 7.64\std{0.01} & 134.8\std{12.4} & 114.8\std{10.5} \\
          & NELBO & \underline{0.685}\std{0.020} & 40.46\std{0.24} & 7.66\std{0.01} & 158.8\std{7.0} & 128.6\std{7.4} \\
          & Bregman & 0.594\std{0.018} & 51.02\std{0.50} & 7.74\std{0.01} & 119.8\std{4.7} & 98.2\std{6.9} \\
          & \textbf{SCE (ours)} & \textbf{0.690}\std{0.010} & \textbf{38.85}\std{0.50} & 7.67\std{0.01} & \textbf{177.2}\std{11.3} & \textbf{137.8}\std{7.6} \\
        \bottomrule
    \end{tabular}
\end{table}

\textbf{Scaled Cross-Entropy is the strongest loss among the considered training objectives.} We compare the standard cross-entropy (CE) \citep{gat2024discrete}, NELBO \citep{sahoo2024simple}, Bregman divergence \citep{lipman2024flowmatchingguidecode}, and our Scaled Cross-Entropy (SCE), with results reported in \cref{tab:loss_functions}. SCE delivers the best overall performance across the considered settings. Its advantage is especially pronounced for the uniform source, where it substantially improves both MAUVE and generation perplexity on OWT relative to all alternative losses, while preserving reasonable entropy. On QM9, it also achieves the highest number of valid molecules and remains competitive in novelty. 
This suggests that context-weighted loss reweighting is a more effective training signal than the commonly used alternatives.

\begin{table}[ht]
    \centering
    \caption{Mixture ($p_t$) vs. neighbor ($\tilde{p}_t$) path comparison.}
    \label{tab:path_ablation}
    \begin{tabular}{ll|ccc|cc}
        \toprule
        \multicolumn{2}{c|}{} & \multicolumn{3}{c}{OWT} & \multicolumn{2}{c}{QM9} \\
        \cmidrule(lr){3-5}\cmidrule(lr){6-7}
        Source & Path & MAUVE $(\uparrow)$ & Gen.\ PPL $(\downarrow)$ & Entropy $(\uparrow)$ & Valid $(\uparrow)$ & Novel $(\uparrow)$ \\
        \midrule
         \multirow{2}{*}{Uniform} & $p_t$ & \textbf{0.777}\std{0.017} & \textbf{30.20}\std{0.22} & 7.66\std{0.01} & 556.0\std{16.1} & 297.6\std{10.1} \\
         & $\tilde{p}_t$ & 0.768\std{0.017} & 37.45\std{0.38} & 7.70\std{0.01} & \textbf{658.4}\std{20.3} & \textbf{310.2}\std{19.6}  \\
         \midrule
         \multirow{2}{*}{Mask} & $p_t$ & 0.690\std{0.010} & 38.85\std{0.50} & 7.67\std{0.01} & 177.2\std{11.3} & 137.8\std{7.6} \\
         & $\tilde{p}_t$ & \textbf{0.751}\std{0.012} & \textbf{36.25}\std{0.52}& 7.65\std{0.01} & \textbf{298.0}\std{13.6} & \textbf{181.6}\std{9.62}  \\ 
         \bottomrule
    \end{tabular}
\end{table}

\textbf{Training with the context-weighted path is more effective for low-data settings.} We next compare training on the standard mixture path $p_t$ and the context-weighted path $\tilde{p}_t$, while using SCE as loss function in both cases; the results are shown in \cref{{tab:path_ablation}}. On OWT with the uniform source, replacing the mixture $p_t$ with the neighbor $\tilde{p}_t$ path does not improve MAUVE or generation perplexity, although it yields a modest increase in entropy. In contrast, the context-weighted path is substantially more beneficial in the masked-source setting, where it improves both MAUVE and generation perplexity. The gains are even more pronounced on QM9, where $\tilde{p}_t$ improves both validity and novelty for both source distributions, with especially large improvements for the masked source, \textbf{nearly doubling the number of valid generated molecules.}

\textbf{An intermediate neighborhood radius provides the best quality--diversity trade-off.}
We ablate the vicinity radius $r$ used in the Neighbor-weighted scaling $\alpha^i(x)$ under the SCE objective (see \cref{fig:vicinity_size_abl} in Appendix). As $r$ increases, both generative perplexity and entropy decrease, indicating more confident but less diverse predictions. MAUVE peaks at $r=3$, suggesting that moderate context aggregation best balances coherence and diversity.

\section{Conclusion}
\label{sec:conclusion}

We studied local context in Discrete Flow Matching. Our results suggest three main takeaways. First, context-dependent sampling provides a cost-efficient inference-time improvement, orthogonal to predictor-corrector sampling techniques. Second, the context-weighted loss substantially improves training and closes much of the gap to semi-autoregressive baselines while preserving any-order generation. Finally, context-weighted paths are particularly useful in constrained and low-data regimes. Overall, local context appears to be an important factor in both training and sampling for Discrete Flow Matching. 

Several limitations remain. Neighbor-weighted sampling is less effective at very low NFE, and our experiments focus on small models and one-dimensional discrete sequences, leaving open how local-context weighting behaves for larger models, longer contexts, and higher-dimensional data such as images. Future work should develop better proxies for token-level uncertainty and use them for more adaptive probability path design, while further narrowing the gap to strong autoregressive models.

\bibliographystyle{unsrtnat}
\bibliography{bibliography}

@inproceedings{
    sahoo2024simple,
    title={Simple and Effective Masked Diffusion Language Models},
    author={Subham Sekhar Sahoo and Marianne Arriola and Aaron Gokaslan and Edgar Mariano Marroquin and Alexander M Rush and Yair Schiff and Justin T Chiu and Volodymyr Kuleshov},
    booktitle={The Thirty-eighth Annual Conference on Neural Information Processing Systems},
    year={2024},
    url={https://openreview.net/forum?id=L4uaAR4ArM}
}

@inproceedings{
    kim2025train,
    title={Train for the Worst, Plan for the Best: Understanding Token Ordering in Masked Diffusions},
    author={Jaeyeon Kim and Kulin Shah and Vasilis Kontonis and Sham M. Kakade and Sitan Chen},
    booktitle={Forty-second International Conference on Machine Learning},
    year={2025},
    url={https://openreview.net/forum?id=DjJmre5IkP}
}

@inproceedings{
    gat2024discrete,
    title={Discrete Flow Matching},
    author={Itai Gat and Tal Remez and Neta Shaul and Felix Kreuk and Ricky T. Q. Chen and Gabriel Synnaeve and Yossi Adi and Yaron Lipman},
    booktitle={The Thirty-eighth Annual Conference on Neural Information Processing Systems},
    year={2024},
    url={https://openreview.net/forum?id=GTDKo3Sv9p}
}

@misc{lipman2024flowmatchingguidecode,
      title={Flow Matching Guide and Code}, 
      author={Yaron Lipman and Marton Havasi and Peter Holderrieth and Neta Shaul and Matt Le and Brian Karrer and Ricky T. Q. Chen and David Lopez-Paz and Heli Ben-Hamu and Itai Gat},
      year={2024},
      eprint={2412.06264},
      archivePrefix={arXiv},
      primaryClass={cs.LG},
      url={https://arxiv.org/abs/2412.06264}, 
}

@misc{lozhkov2024fineweb-edu,
    author       = { Lozhkov, Anton and Ben Allal, Loubna and von Werra, Leandro and Wolf, Thomas },  
    title        = { FineWeb-Edu: the Finest Collection of Educational Content }, 
    year         = 2024,  
    url          = { https://huggingface.co/datasets/HuggingFaceFW/fineweb-edu },  
    doi          = { 10.57967/hf/2497 },
    publisher    = { Hugging Face }
}

@article{song2025seed,
  title={Seed diffusion: A large-scale diffusion language model with high-speed inference},
  author={Song, Yuxuan and Zhang, Zheng and Luo, Cheng and Gao, Pengyang and Xia, Fan and Luo, Hao and Li, Zheng and Yang, Yuehang and Yu, Hongli and Qu, Xingwei and others},
  journal={arXiv preprint arXiv:2508.02193},
  year={2025}
}

@article{gat2025set,
  title={Set block decoding is a language model inference accelerator},
  author={Gat, Itai and Ben-Hamu, Heli and Havasi, Marton and Haziza, Daniel and Reizenstein, Jeremy and Synnaeve, Gabriel and Lopez-Paz, David and Karrer, Brian and Lipman, Yaron},
  journal={arXiv preprint arXiv:2509.04185},
  year={2025}
}

@article{li2025autoregressive,
  title={Autoregressive image generation with randomized parallel decoding},
  author={Li, Haopeng and Yang, Jinyue and Li, Guoqi and Wang, Huan},
  journal={arXiv preprint arXiv:2503.10568},
  year={2025}
}

@article{zhu2025llada,
  title={LLaDA 1.5: Variance-Reduced Preference Optimization for Large Language Diffusion Models},
  author={Zhu, Fengqi and Wang, Rongzhen and Nie, Shen and Zhang, Xiaolu and Wu, Chunwei and Hu, Jun and Zhou, Jun and Chen, Jianfei and Lin, Yankai and Wen, Ji-Rong and others},
  journal={arXiv preprint arXiv:2505.19223},
  year={2025}
}

@inproceedings{nie2025large,
  title={Large Language Diffusion Models},
  author={Nie, Shen and Zhu, Fengqi and You, Zebin and Zhang, Xiaolu and Ou, Jingyang and Hu, Jun and ZHOU, JUN and Lin, Yankai and Wen, Ji-Rong and Li, Chongxuan},
  booktitle={ICLR 2025 Workshop on Deep Generative Model in Machine Learning: Theory, Principle and Efficacy},
  year={2025}
}

@article{ben2025accelerated,
  title={Accelerated Sampling from Masked Diffusion Models via Entropy Bounded Unmasking},
  author={Ben-Hamu, Heli and Gat, Itai and Severo, Daniel and Nolte, Niklas and Karrer, Brian},
  journal={arXiv preprint arXiv:2505.24857},
  year={2025}
}

@article{he2025neighboring,
  title={Neighboring autoregressive modeling for efficient visual generation},
  author={He, Yefei and He, Yuanyu and He, Shaoxuan and Chen, Feng and Zhou, Hong and Zhang, Kaipeng and Zhuang, Bohan},
  journal={arXiv preprint arXiv:2503.10696},
  year={2025}
}

@article{campbell2022continuous,
  title={A continuous time framework for discrete denoising models},
  author={Campbell, Andrew and Benton, Joe and De Bortoli, Valentin and Rainforth, Thomas and Deligiannidis, George and Doucet, Arnaud},
  journal={Advances in Neural Information Processing Systems},
  volume={35},
  pages={28266--28279},
  year={2022}
}

@article{arriola2025block,
  title={Block diffusion: Interpolating between autoregressive and diffusion language models},
  author={Arriola, Marianne and Gokaslan, Aaron and Chiu, Justin T and Yang, Zhihan and Qi, Zhixuan and Han, Jiaqi and Sahoo, Subham Sekhar and Kuleshov, Volodymyr},
  journal={arXiv preprint arXiv:2503.09573},
  year={2025}
}

@inproceedings{
    lou2024discrete,
    title={Discrete Diffusion Modeling by Estimating the Ratios of the Data Distribution},
    author={Aaron Lou and Chenlin Meng and Stefano Ermon},
    booktitle={Forty-first International Conference on Machine Learning},
    year={2024},
    url={https://openreview.net/forum?id=CNicRIVIPA}
}

@misc{merity2016pointer,
      title={Pointer Sentinel Mixture Models},
      author={Stephen Merity and Caiming Xiong and James Bradbury and Richard Socher},
      year={2016},
      eprint={1609.07843},
      archivePrefix={arXiv},
      primaryClass={cs.CL}
}

@article{pillutla2021mauve,
  title={Mauve: Measuring the gap between neural text and human text using divergence frontiers},
  author={Pillutla, Krishna and Swayamdipta, Swabha and Zellers, Rowan and Thickstun, John and Welleck, Sean and Choi, Yejin and Harchaoui, Zaid},
  journal={Advances in Neural Information Processing Systems},
  volume={34},
  pages={4816--4828},
  year={2021}
}

@misc{Gokaslan2019OpenWeb,  
	title={OpenWebText Corpus},
	author={Aaron Gokaslan and Vanya Cohen},
	howpublished={\url{http://Skylion007.github.io/OpenWebTextCorpus}}, 
	year={2019}
}

@inproceedings{
    loshchilov2018decoupled,
    title={Decoupled Weight Decay Regularization},
    author={Ilya Loshchilov and Frank Hutter},
    booktitle={International Conference on Learning Representations},
    year={2019},
    url={https://openreview.net/forum?id=Bkg6RiCqY7},
}

@inproceedings{shaulflow,
  title={Flow Matching with General Discrete Paths: A Kinetic-Optimal Perspective},
  author={Shaul, Neta and Gat, Itai and Havasi, Marton and Severo, Daniel and Sriram, Anuroop and Holderrieth, Peter and Karrer, Brian and Lipman, Yaron and Chen, Ricky TQ},
  booktitle={The Thirteenth International Conference on Learning Representations},
  year={2025}
}

@misc{bie2026llada21speedingtextdiffusion,
      title={LLaDA2.1: Speeding Up Text Diffusion via Token Editing}, 
      author={Tiwei Bie and Maosong Cao and Xiang Cao and Bingsen Chen and Fuyuan Chen and Kun Chen and Lun Du and Daozhuo Feng and Haibo Feng and Mingliang Gong and Zhuocheng Gong and Yanmei Gu and Jian Guan and Kaiyuan Guan and Hongliang He and Zenan Huang and Juyong Jiang and Zhonghui Jiang and Zhenzhong Lan and Chengxi Li and Jianguo Li and Zehuan Li and Huabin Liu and Lin Liu and Guoshan Lu and Yuan Lu and Yuxin Ma and Xingyu Mou and Zhenxuan Pan and Kaida Qiu and Yuji Ren and Jianfeng Tan and Yiding Tian and Zian Wang and Lanning Wei and Tao Wu and Yipeng Xing and Wentao Ye and Liangyu Zha and Tianze Zhang and Xiaolu Zhang and Junbo Zhao and Da Zheng and Hao Zhong and Wanli Zhong and Jun Zhou and Junlin Zhou and Liwang Zhu and Muzhi Zhu and Yihong Zhuang},
      year={2026},
      eprint={2602.08676},
      archivePrefix={arXiv},
      primaryClass={cs.LG},
      url={https://arxiv.org/abs/2602.08676}, 
}

@misc{bie2025llada20scalingdiffusionlanguage,
      title={LLaDA2.0: Scaling Up Diffusion Language Models to 100B}, 
      author={Tiwei Bie and Maosong Cao and Kun Chen and Lun Du and Mingliang Gong and Zhuochen Gong and Yanmei Gu and Jiaqi Hu and Zenan Huang and Zhenzhong Lan and Chengxi Li and Chongxuan Li and Jianguo Li and Zehuan Li and Huabin Liu and Ling Liu and Guoshan Lu and Xiaocheng Lu and Yuxin Ma and Jianfeng Tan and Lanning Wei and Ji-Rong Wen and Yipeng Xing and Xiaolu Zhang and Junbo Zhao and Da Zheng and Jun Zhou and Junlin Zhou and Zhanchao Zhou and Liwang Zhu and Yihong Zhuang},
      year={2025},
      eprint={2512.15745},
      archivePrefix={arXiv},
      primaryClass={cs.LG},
      url={https://arxiv.org/abs/2512.15745}, 
}

@inproceedings{
    zheng2024a,
    title={A Reparameterized Discrete Diffusion Model for Text Generation},
    author={Lin Zheng and Jianbo Yuan and Lei Yu and Lingpeng Kong},
    booktitle={First Conference on Language Modeling},
    year={2024},
    url={https://openreview.net/forum?id=PEQFHRUFca}
}

@article{ramakrishnan2014quantum,
  title={Quantum chemistry structures and properties of 134 kilo molecules},
  author={Ramakrishnan, Raghunathan and Dral, Pavlo O and Rupp, Matthias and Von Lilienfeld, O Anatole},
  journal={Scientific data},
  volume={1},
  number={1},
  pages={1--7},
  year={2014},
  publisher={Nature Publishing Group}
}

@article{radford2019language,
  title={Language models are unsupervised multitask learners},
  author={Radford, Alec and Wu, Jeffrey and Child, Rewon and Luan, David and Amodei, Dario and Sutskever, Ilya and others},
  journal={OpenAI blog},
  volume={1},
  number={8},
  pages={9},
  year={2019}
}

@inproceedings{gu2022vector,
  title={Vector quantized diffusion model for text-to-image synthesis},
  author={Gu, Shuyang and Chen, Dong and Bao, Jianmin and Wen, Fang and Zhang, Bo and Chen, Dongdong and Yuan, Lu and Guo, Baining},
  booktitle={Proceedings of the IEEE/CVF conference on computer vision and pattern recognition},
  pages={10696--10706},
  year={2022}
}

@inproceedings{vignacdigress,
  title={DiGress: Discrete Denoising diffusion for graph generation},
  author={Vignac, Clement and Krawczuk, Igor and Siraudin, Antoine and Wang, Bohan and Cevher, Volkan and Frossard, Pascal},
  booktitle={The Eleventh International Conference on Learning Representations},
  year={2023}
}

@inproceedings{kelviniuswyckoffdiff,
  title={WyckoffDiff--A Generative Diffusion Model for Crystal Symmetry},
  author={Kelvinius, Filip Ekstr{\"o}m and Andersson, Oskar B and Parackal, Abhijith S and Qian, Dong and Armiento, Rickard and Lindsten, Fredrik},
  booktitle={Forty-second International Conference on Machine Learning},
  year={2025}
}

@inproceedings{besnierhalton,
  title={Halton Scheduler for Masked Generative Image Transformer},
  author={Besnier, Victor and Chen, Mickael and Hurych, David and Valle, Eduardo and Cord, Matthieu},
  booktitle={The Thirteenth International Conference on Learning Representations},
  year={2025}
}

@inproceedings{paperno2016lambada,
  title={The LAMBADA dataset: Word prediction requiring a broad discourse context},
  author={Paperno, Denis and Kruszewski, Germ{\'a}n and Lazaridou, Angeliki and Pham, Ngoc-Quan and Bernardi, Raffaella and Pezzelle, Sandro and Baroni, Marco and Boleda, Gemma and Fern{\'a}ndez, Raquel},
  booktitle={Proceedings of the 54th annual meeting of the association for computational linguistics (volume 1: Long papers)},
  pages={1525--1534},
  year={2016}
}


\appendix

\section{Theoretical results}
\label{sec:theory}

\subsection{Discrete Flow Matching background}
\label{sec:dfm_background}

We consider discrete sequences $x \in \mathcal{D}^N$ of length $N$, where $\mathcal{D}$ is a finite vocabulary. We denote the $i$-th coordinate of $x$ by $x^i$ and the sequence without the $i$-th coordinate by $x^{\overline{i}}$. For $a \in \mathcal{D}$, we write $\delta(x^i, a) = \mathbf{1}[x^i = a]$ for a Dirac delta function. We consider independent coupling $(X_0, X_1)\sim p(X_0)q(X_1)$, where $X_0$ is sampled from a simple source noise distribution $p(X_0)$ and $X_1 \sim p_{\text{data}}$ is a data point.

Discrete Flow Matching~\citep{gat2024discrete} defines a probability path $p_t(x \mid x_0, x_1)$ between $x_0$ and $x_1$, and models its evolution via a Continuous-Time Markov Chain (CTMC). The dynamics are governed by the Kolmogorov forward equation:
$$
\frac{d}{dt}p_t(x \mid x_0, x_1)
=
\sum_{z \in \mathcal{D}^N} \sum_{i=1}^N
\delta(x^{\overline{i}}, z^{\overline{i}})
\, p_t(z \mid x_0, x_1)\,
u_t^i(x^i, z \mid x_0, x_1),
$$
where $u_t^i(x^i, z \mid x_0, x_1)$ denotes the coordinate-wise probability velocity, specifying the rate of transition at position $i$.

One key aspect of the standard DFM setup is that the velocity factorizes over coordinates:
$$
u_t(x, z \mid x_0, x_1)
=
\sum_{i=1}^N u_t^i(x^i, z \mid x_0, x_1).
$$
Importantly, this factorization does not imply that the distribution $p_t(x \mid x_0, x_1)$ factorizes across coordinates. However, it enables efficient coordinate-wise sampling, as the CTMC can be simulated using an Euler discretization, updating the tokens independently of each other.

A common choice in DFM is the convex mixture  path, defined independently for each coordinate as
$$
p_t(x^i \mid x_0, x_1)
=
(1-\kappa_t)\delta(x^i, x_0^i)
+
\kappa_t \delta(x^i, x_1^i),
$$
where $\kappa_t$ is a monotonically increasing scheduler. The corresponding coordinate-wise velocity is then given by
$$
u^i_t(x^i, z \mid x_0, x_1)
=
\frac{\dot{\kappa}_t}{1 - \kappa_t}
\left[
\delta(x^i, x_1^i) - \delta(x^i, z)
\right].
$$
This construction leads to a simple and efficient training and sampling procedure, but treats all coordinates uniformly, regardless of their conditioning context. 

At inference time, DFM simulates the learned CTMC on a grid
$0=t_0<t_1<\dots<t_K=1$. Given the current state $X_t$, the model first samples a clean-token proposal for each coordinate,
\[
X_1^i \sim p_{1\mid t}^{\theta,i}(\cdot \mid X_t).
\]
It then applies an Euler step using the conditional velocity obtained by replacing the data endpoint $x_1^i$ in the convex-path velocity with the sampled proposal $Y^i$:
\[
X_{t+h}^i
\sim
\delta_{X_t^i}(\cdot)
+
h\,u_t^i(\cdot, X_t \mid X_0, X_1),
\]
Thus the Euler step either keeps coordinate $i$ unchanged or moves it toward the model-sampled clean token $X_1^i$ with rate $\dot{\kappa}_t/(1-\kappa_t)$.


\subsection{Boundary conditions and marginals}
\label{sec:boundary_conditions}

We prove Theorem 1 in three steps. First, we show that under the standard convex path, the number of unmasked tokens follows a pure-birth CTMC with binomial marginals. Second, we show that the normalization condition on $\alpha$ makes the context-weighted path induce the same CTMC for this count process, which gives the endpoint constraints. Finally, under permutation equivariance, the unmasked coordinates are exchangeable, so matching the count distribution also gives the correct per-coordinate marginals.

For a path $X_t$, define the number of unmasked tokens
$$
M_t = \sum_{i=1}^N \mathbf{1}[X_t^i = x_1^i].
$$
Under the standard convex path, each coordinate is unmasked independently with probability $\kappa_t$, hence
$$
M_t \sim \mathrm{Binomial}(N,\kappa_t).
$$
The following proposition describes the corresponding count process.





\begin{restatable}{proposition}{MCTMC}
\label{prop:m_ctmc}
$M_t$ is a CTMC 
$$
\mathbb{P}(M_{t+h} = y \mid M_t = x) = \delta(y, x) + hq(y, x) + o(h^2)
$$
    
where
$$q(y, x) = 
\begin{cases}
    (N - x) \frac{\dot{\kappa}_t}{1-\kappa_t}, & y = x+1 \\
    -(N - x) \frac{\dot{\kappa}_t}{1-\kappa_t}, & y = x \\
    0, & \text{otherwise}
\end{cases}
$$
\end{restatable} 

\begin{proof}
Let's derive the transition probabilities from $p_t(m) = \mathbb{P}(M_t = m)$:
\begin{align*}
    \frac{d}{dt}p_t(m) &= \frac{d}{dt} \binom{N}{m} (\kappa_t)^{m} (1-\kappa_t)^{N-m} \\
    & = m \binom{N}{m} (\kappa_t)^{m-1} (1-\kappa_t)^{N-m} \dot{\kappa}_t + (N-m) \binom{N}{m}  (\kappa_t)^{m} (1-\kappa_t)^{N-m-1}(-\kappa_t) \\
    & = \frac{\dot{\kappa}_t}{1-\kappa_t} m \binom{N}{m} (\kappa_t)^{m-1} (1-\kappa_t)^{N-(m-1)} - \frac{\dot{\kappa}_t}{1-\kappa_t} (N-m) \binom{N}{m}  (\kappa_t)^{m} (1-\kappa_t)^{N-m} \\
    & = \frac{\dot{\kappa}_t}{1-\kappa_t}(N-(m-1)) \binom{N}{m-1} (\kappa_t)^{m-1} (1-\kappa_t)^{N-(m-1)} - \frac{\dot{\kappa}_t}{1-\kappa_t} (N-m) p_t(m) \\
    &= \frac{\dot{\kappa}_t}{1-\kappa_t}(N-(m-1))p_t(m-1)  - \frac{\dot{\kappa}_t}{1-\kappa_t} (N-m) p_t(m)
\end{align*}

Note the transition probabilities
\begin{align*}
    q_t(m, m-1) &= \frac{\dot{\kappa}_t}{1-\kappa_t}(N-(m-1)) \\
    q_t(m, m) &= -\frac{\dot{\kappa}_t}{1-\kappa_t}(N-m)
\end{align*}

and so

\begin{align*}
    \frac{d}{dt}p_t(m) = q_t(m, m-1)p_t(m-1) + q(m, m)p_t(m) = \sum_{z=0}^N q(m, z)p_t(z)
\end{align*}

which is exactly the Kolmogorov forward equation with $q(y, x)$.
\end{proof}



We next show that the context-weighted path has the same count process whenever the weights are normalized over the masked coordinates. This is the key property needed for the endpoint constraints: if $\tilde M_t$ has the same law as $M_t$, then $\tilde M_0=0$ and $\tilde M_1=N$ almost surely.

\begin{restatable}{proposition}{normcondition}
\label{prop:norm_condition}
A sufficient condition for $\tilde{M}_t \sim \text{Binomial}(N, \kappa_t)$ is:

$$
\sum_{i = 1}^N \alpha^i(x_t) \cdot \delta(x^i_t, x^i_0) = \sum_{i=1}^N \delta(x^i_t, x^i_0)
$$
\end{restatable}
\begin{proof}
First, note that probability of $M_t$ not changing in $[t, t+h)$ is the same as probability of $X_t$ not changing:
\begin{align*}
    \mathbb{P}(X_{t+h} = x \mid X_{t} = x, X_0 = x_0, X_1 = x_1) &= 1 + h \sum_{i=1}^N u^i_t(x^i, x \mid x_0, x_1) + o(h) \\
    & = 1 + h\sum_{i=1}^N \frac{\dot{\kappa}_t}{1-\kappa_t} \left[ \delta(x^i, x_1^i) - 1 \right] + o(h) \\
    & = 1 - h \frac{\dot{\kappa}_t}{1-\kappa_t} (N-\sum_{i=1}^N \delta(x^i, x^i_1)) + o(h) \\
    & = 1 - h \frac{\dot{\kappa}_t}{1-\kappa_t} (N-m) + o(h) \\
    & = \mathbb{P}(M_{t+h} = m \mid M_t = m)
\end{align*}

which makes sense since $x^i_1$ is absorbing for each token, so no cyclic movements ($x_0^i \to x_1^i \to x_0^i$) are allowed. Now, if we plug in $v^i_t$ instead of $u^i_t$ in the same derivation

\begin{align*}
    \mathbb{P}(\tilde{X}_{t+h} = x \mid \tilde{X}_{t} = x, \tilde{X}_0 = x_0, \tilde{X}_1 = x_1) &= 1 + h \sum_{i=1}^N v^i_t(x^i, x \mid x_0, x_1) + o(h) \\
        & = 1 - h \frac{\dot{\kappa}_t}{1-\kappa_t}  \sum_{i=1}^N \alpha^i(x) \delta(x^i, x^i_0) + o(h) \\
\end{align*}

so if $\sum_{i=1}^N\alpha^i(x) \delta(x^i, x^i_0) = \sum_{i=1}^N \delta(x^i, x^i_0) = N-m$, we have 

$$
P(\tilde{X}_{t+h} = x \mid \tilde{X}_{t} = x, \tilde{X}_0 = x_0, \tilde{X}_1 = x_1) = \mathbb{P}(X_{t+h} = x \mid X_{t} = x, X_0 = x_0, X_1 = x_1)
$$

and so $\tilde{M}_t$ is following the same CTMC as $M_t$.
\end{proof}

Thus the normalization condition preserves the distribution of the number of unmasked tokens. Since $\kappa_0=0$ and $\kappa_1=1$, it follows that $\tilde M_0=0$ and $\tilde M_1=N$ almost surely, which implies
\[
\tilde p_0(x\mid x_0,x_1)=\delta(x,x_0),
\qquad
\tilde p_1(x\mid x_0,x_1)=\delta(x,x_1).
\]

Equivalently, the normalization condition requires the weights to average to one over the currently masked coordinates:
\[
\frac{1}{|\{i \colon z^i = x_0^i\}|}
\sum_{i \colon z^i = x_0^i} \alpha^i(z) = 1.
\]
This controls only the total number of updates, not which coordinates receive them. For example, a positional bias such as $\alpha^i(z)\propto (N-i)$ can preserve the total rate while making left-most coordinates transition earlier. To recover the standard per-coordinate marginals, we additionally require the weighting rule to treat positions symmetrically. Let $\operatorname{shift}_k(x)$ denote a circular shift of $x$ by $k$ positions, i.e.
\[
\operatorname{shift}_k(x)^i = x^{i-k \!\!\!\pmod N}.
\]

\begin{restatable}{proposition}{shiftequiv}
\label{prop:shift_equiv}
Let $B_t \in \{0,1\}^N$ denote the binary reveal mask, where
$B_t^i=1$ if coordinate $i$ has transitioned to $x_1^i$, and
$B_t^i=0$ otherwise. Suppose that the context weights depend only on
the reveal mask, i.e., there exists a function
\[
\bar{\alpha}:\{0,1\}^N \to \mathbb{R}_+^N
\]
such that
\[
\alpha^i(\tilde X_t)=\bar{\alpha}^i(B_t)
\qquad \forall i.
\]
Assume that $\bar{\alpha}$ is equivariant under circular shifts:
\[
\bar{\alpha}^{i+k}\bigl(\operatorname{shift}_k(b)\bigr)
=
\bar{\alpha}^i(b)
\qquad
\forall b\in\{0,1\}^N,\ \forall i,k,
\]
with indices taken modulo $N$, and that the normalization condition in
\cref{prop:norm_condition} holds. Then
\[
\mathbb{P}(B_t^i=1\mid x_0,x_1)= \mathbb{P}(\tilde{X}_t^i=x_1^i\mid x_0,x_1) = \kappa_t
\qquad \forall i.
\]
\end{restatable}

\begin{proof}
Let
\[
S_t=\{i:B_t^i=1\}
\]
be the set of revealed coordinates, so that
\[
|S_t|=\tilde M_t.
\]
Since $\bar{\alpha}$ is equivariant under circular shifts, shifting the
reveal mask and the corresponding coordinate by the same offset
preserves all transition rates. Since $B_0=\mathbf{0}$ is invariant
under circular shifts, the law of $B_t$ is also invariant under
circular shifts. Hence, for every offset $k$,
\[
\mathbb{P}(S_t=S\mid x_0,x_1)
=
\mathbb{P}
\bigl(S_t=\operatorname{shift}_k(S)\mid x_0,x_1\bigr).
\]

Thus all coordinates have the same marginal probability of being
revealed. Indeed, for any $i,j$, there exists a circular shift mapping
$i$ to $j$, and therefore
\[
\mathbb{P}(i\in S_t\mid x_0,x_1)
=
\mathbb{P}(j\in S_t\mid x_0,x_1).
\]
Denote this common probability by $p_t$. Then
\[
\mathbb{E}[\tilde M_t\mid x_0,x_1]
=
\mathbb{E}[|S_t|\mid x_0,x_1]
=
\sum_{i=1}^N
\mathbb{P}(i\in S_t\mid x_0,x_1)
=
Np_t.
\]

By \cref{prop:norm_condition}, the normalization condition implies
\[
\tilde M_t\sim\mathrm{Binomial}(N,\kappa_t),
\]
and hence
\[
\mathbb{E}[\tilde M_t\mid x_0,x_1]=N\kappa_t.
\]
Combining the two equalities gives
\[
p_t=\kappa_t.
\]
Therefore,
\[
\mathbb{P}(\tilde X_t^i=x_1^i\mid x_0,x_1)=\kappa_t.
\]
\end{proof}

Now, we are ready to prove the theorem 1.

\begin{proof}[Proof of \cref{thrm:main}]
By \cref{prop:norm_condition}, the normalization condition implies that $\tilde M_t$ has the same law as the count process of the standard convex path, namely $\mathrm{Binomial}(N,\kappa_t)$. Since $\kappa_0=0$ and $\kappa_1=1$, this gives the endpoint constraints. If $\alpha$ is additionally equivariant under circular shifts, \cref{prop:shift_equiv} gives the coordinate-wise marginal

\[
\mathbb{P}(\tilde X_t^i=x_1^i\mid x_0,x_1)=\kappa_t
\]

for every coordinate $i$.

\end{proof}

\subsection{Correctness of sampling}
\label{sec:alg_correct}

\paragraph{Solver normalization and marginals.}
Both Neighbor-weighted and Entropy-weighted solvers are normalized over the currently masked coordinates. If $\tilde{\alpha}^i(x_t)$ denotes the raw score, we set
\[
\alpha^i(x_t)
\propto
\exp(s\tilde{\alpha}^i(x_t)),
\qquad i \in \mathcal M_t=\{i:x_t^i=x_0^i\},
\]
and choose the proportionality constant so that
\[
\frac{1}{|\mathcal M_t|}\sum_{i\in\mathcal M_t}\alpha^i(x_t)=1.
\]
Therefore both solvers satisfy the normalization condition in \cref{prop:norm_condition} and preserve the distribution of the total number of unmasked tokens.

For Neighbor-weighted, if local windows are computed with circular padding, the rule is shift-equivariant: shifting the sequence shifts all neighbor counts by the same amount. Hence all positions are treated symmetrically, and the per-coordinate marginals are preserved by \cref{prop:shift_equiv}. 

For Entropy-weighted, the normalization still preserves the total number of updates, and therefore the endpoint constraints. However, exact circular-shift equivariance of the model predictions is not guaranteed:
\[
H\!\left(p^{\theta, (i+k \bmod N)}(\cdot \mid \operatorname{shift}_k(x_t))\right)
\neq
H\!\left(p^{\theta, i}(\cdot \mid x_t)\right)
\]
in general. Therefore, Entropy-weighted sampling is not covered by the per-coordinate marginal preservation result: the model may induce position- or value-dependent preferences over the reveal order.

\paragraph{Exactness of train-time path sampling.}
\cref{alg:path_sampling} avoids simulating the full continuous-time trajectory by separating the process into the number of jumps and their identities. By \cref{prop:norm_condition}, the number of unmasked tokens at time $t$ is preserved:
\[
\tilde M_t \sim \mathrm{Binomial}(N,\kappa_t).
\]
Thus we can first sample $m \sim \mathrm{Binomial}(N,\kappa_t)$. It remains to sample which $m$ coordinates are unmasked. Under the context-weighted CTMC, if the current state is $x$, a masked coordinate $i$ has transition rate
\[
\alpha^i(x)\frac{\dot\kappa_t}{1-\kappa_t}.
\]
Therefore, conditioned on a jump occurring, the next coordinate is selected with probability
\[
\frac{\alpha^i(x)}
{\sum_{j:x^j=x_0^j}\alpha^j(x)},
\]
since the time-dependent factor cancels. This is exactly the sequential selection rule used in \cref{alg:path_sampling}. Hence the algorithm samples the same time marginal $\tilde p_t$ as the continuous-time process.

\subsection{Evidence Lower Bound (ELBO)}

In the case of the mixture path, \citet{shaulflow} presents a formula for evidence lower bound (ELBO) 

$$
\log p_1^{\theta}(x_1) \geq - \mathbb{E}_{t, X_0, X_t \sim p_{t \mid 0, 1}} \sum_i D\left(u^i_t(\cdot, X^i_t \mid X_0, x_1), u^{\theta, i}_t(\cdot, X_t)\right)
$$

where 

$$
D\left(u^i_t(\cdot, x^i \mid x_0, x_1), u^{\theta, i}_t(\cdot, x)\right) = \frac{\dot{\kappa}_t}{1-\kappa_t} \left[\left(\delta(x_1^i, x^i) - 1\right) \log p_{1 \mid t}^{\theta, i}(x^i_1 \mid x) + \delta(x^i_1, x^i) - p_{1 \mid t}^{\theta, i}(x^i \mid x)\right]
$$

\begin{restatable}{proposition}{elbo}
\label{prop:elbo}
Consider the context-weighted velocity
\[
v_t^i(\cdot,x \mid x_0,x_1)
=
\alpha^i(x)u_t^i(\cdot,x \mid x_0,x_1),
\]
where $\alpha^i(x)\ge 0$ does not depend on the target state. Let the learned marginal velocity be parameterized as
\[
v_t^{\theta,i}(\cdot,x)=\alpha^i(x)u_t^{\theta,i}(\cdot,x).
\]
For the CTMC rate divergence
\[
D(a,b)=a\log\frac{a}{b}-a+b,
\]
applied coordinate-wise over transition rates, we have
\[
D\!\left(v_t^i(\cdot,x \mid x_0,x_1),v_t^{\theta,i}(\cdot,x)\right)
=
\alpha^i(x)
D\!\left(u_t^i(\cdot,x \mid x_0,x_1),u_t^{\theta,i}(\cdot,x)\right).
\]
Thus the corresponding NELBO term is reweighted by $\alpha^i(x)$.
\end{restatable}

\begin{proof}
Since $v_t^i=\alpha^i(x)u_t^i$ and $v_t^{\theta,i}=\alpha^i(x)u_t^{\theta,i}$, it is enough to use the positive homogeneity of the CTMC rate divergence. For scalar rates,
\[
D(\alpha a,\alpha b)
=
\alpha a\log\frac{\alpha a}{\alpha b}-\alpha a+\alpha b
=
\alpha\left(a\log\frac{a}{b}-a+b\right)
=
\alpha D(a,b).
\]
Summing over possible target states preserves the same factor, since $\alpha^i(x)$ is independent of the target state. Therefore,
\[
D\!\left(v_t^i(\cdot,x \mid x_0,x_1),v_t^{\theta,i}(\cdot,x)\right)
=
\alpha^i(x)
D\!\left(u_t^i(\cdot,x \mid x_0,x_1),u_t^{\theta,i}(\cdot,x)\right).
\]
\end{proof}

\subsection{Loss Functions Analysis}
\label{sec:loss_func}

The Bregman divergence objective in DFM can be written as
\begin{align*}
    D\left(u^i_t(\cdot, x^i \mid x_0, x_1), u^{\theta, i}_t(\cdot, x)\right) =  \frac{\dot{\kappa}_t}{1-\kappa_t} \left[(\delta(x_1^i, x^i)-1)\log p_{1\mid t}^{\theta,i}(x_1^i \mid x) + \delta(x_1^i, x^i) - p_{1\mid t}^{\theta,i}(x^i \mid x) \right]
\end{align*}

The loss is scaled by the jump coefficient $\dot{\kappa}_t/(1-\kappa_t)$, which depends only on time. For common schedules this coefficient increases as $t$ approaches $1$, giving more weight to later, less-masked states and less weight to heavily masked ones. This captures a coarse notion of prediction difficulty through the amount of global context, but it does not distinguish tokens within the same state. In particular, two masked tokens at the same time step receive the same global scaling even if one has many visible neighbors and the other has little local context. Therefore, Bregman and NELBO-style objectives \citep{sahoo2024simple} still mix token-level prediction problems of different difficulty, motivating an explicit local-context weighting.





\newpage

\section{Experiments}

\subsection{Experimental setup}
\label{sec:app_setup}


\paragraph{Datasets.} We use OpenWebText (OWT) \citep{Gokaslan2019OpenWeb} as our main training dataset to remain comparable with standard text diffusion baselines such as MDLM and BD3LM. OWT contains roughly 8M web documents, making it a suitable large-scale text benchmark. As a smaller alternative from a different domain, we use QM9 \citep{ramakrishnan2014quantum}, which contains about 134k small molecules. We represent QM9 molecules in SMILES format, a linear encoding of molecular graphs, where local token neighborhoods often correspond to meaningful local chemical structure; therefore, locality remains a relevant inductive signal.

\paragraph{Metrics.} For text, we report generative perplexity, computed with GPT-2\citep{radford2019language}, as a proxy for coherence; token entropy, computed with $\log_2$ as in DFM \citep{gat2024discrete}, as a measure of token-level diversity; and MAUVE \citep{pillutla2021mauve} as our main quality--diversity metric. For QM9, we report the number of valid SMILES strings (\textsc{Valid}) and the number of generated molecules that do not appear in the training set (\textsc{Novel}), both measured out of 1024 generated samples. For all metrics except MAUVE, we estimate standard deviations by splitting generated samples into 5 folds. Since MAUVE is a two-sample metric and depends on sample size, we instead estimate its variance using 5 bootstrap resamples. For the data row in \cref{tab:baselines}, we compute MAUVE between two independent subsets of 5000 validation samples. Although the theoretical upper bound is $1.0$ in the infinite-sample limit, we report this finite-sample estimate as a more realistic reference point and a fairer comparison for generated samples of the same size.

\paragraph{Training and inference.} For training models with Discrete Flow Matching, we employ the Transformer model, with the same architecture and setup as in \citet{sahoo2024simple}. On OWT, we train with a total batch size of 512 samples with a fixed length of 1024 for 1 million steps. We use AdamW optimizer \citep{loshchilov2018decoupled}, setting $\beta_1 = 0.9$ and $\beta_2 = 0.999$, with peak learning rate of $3\times 10^{-4}$ and cosine annealing to $3 \times 10^{-6}$, and a warmup of 2.5k steps. The training takes around 4.5 days on 8 NVIDIA H100 GPUs. Training on QM9 took 30 minutes with the same GPU setup.  For testing models trained on OWT, we follow \citet{sahoo2024simple} and generate 5000 samples with nucleus sampling $p = 0.9$. For models trained on QM9, we set $p=1.0$ (no nucleus sampling) and generate 5120 samples in total. See complete setup in \cref{tab:hyperparams}.  

\begin{table}[ht]
    \caption{Training, inference and flow hyperparameters}
    \label{tab:hyperparams}
    \centering
    \begin{tabular}{l|cc}
        \toprule
        Parameter & OWT & QM9 \\
        \midrule
        \# parameters & 169M & 92M \\
        Batch size & 512 & 2048 \\
        Length & 1024 & 32 \\
        Peak LR & $3 \times 10^{-4}$ & $3 \times 10^{-5}$ \\
        Min LR & $3 \times 10^{-6}$ & $3 \times 10^{-7}$ \\
        warmup & 2500 & 1000 \\
        Weight decay & 0.03 & 0.01 \\
        $(\beta_1, \beta_2)$ & (0.9, 0.999) & (0.9, 0.999) \\
        \# training steps & 1000k & 25k \\
        \midrule
        Scheduler $\kappa(t)$ & $t^2$ & $t^2$ \\
        \midrule
        \# gen. samples & 5000 & 5120 \\ 
        Nucleus $p$ & 0.9 & 1.0 \\
        Temperature & 1.0 & 1.0\\
        \bottomrule
    \end{tabular}
\end{table}

\paragraph{Baselines.} 
Our main direct baseline is standard Discrete Flow Matching (DFM) \citep{gat2024discrete}, trained and sampled with the same architecture, data, and evaluation setup as our method. We also compare against autoregressive and semi-autoregressive language modeling baselines, including AR and BD3LM \citep{arriola2025block} with different block lengths $L'$, which provide strong references for generation quality under more structured decoding orders. Finally, we include representative discrete diffusion baselines, SEDD \citep{lou2024discrete} and MDLM \citep{sahoo2024simple}, to compare against non-autoregressive diffusion-style models trained on the same text generation setting.

\subsection{Sweeps}

\paragraph{Inference-time experiments.} For neighbor-weighted and entropy-weighted solvers, we run the sweep on $r \in \{1, 2, 3\}$, $s \in \{0.5, 1.0, 2.0, 3.0, 4.0, 5.0\}$ and $\beta \in \{1.0,2.0,\dots,7.0\}$. 

\begin{table}[ht]
    \caption{Best hyperparameters for Entropy-weighted $(\beta)$ and Neighbor-weighted $(r, s)$ solvers}
    \label{tab:nw_hyperparams}
    \centering
    \begin{tabular}{ll|ccc}
        \toprule
        Source & Dataset & r & s & $\beta$ \\
        \midrule
        \multirow{2}{*}{Uniform} & OWT & 1 & 5.0 & 6.0 \\
        & QM9 & 1 & 4.0 & 7.0 \\
        \midrule
        \multirow{2}{*}{Mask} & OWT & 1 & 5.0 & 6.0 \\
        & QM9 & 1 & 4.0 & 6.0 \\
        \bottomrule
    \end{tabular}
\end{table}

\paragraph{Trained models.} For SCE-trained models, we use only Euler solver, no hyperparameter sweeps. We find that using Neighbor-weighted solver provides no improvement. For CW-DFM ($\tilde{p}_t$), on OWT, we set $r=3$ and run training with $s \in \{0.5, 1.0\}$. We run the sweep on neighbor-weighted solver with the trained model on $r \in \{1, 2, 3\}$ and $s \in \{0.5, 1.0, 2.0\}$.  For training CW-DFM and SCE in QM9, we performed the sweep in $r \in \{1, 2, 3\}$ and $s \in \{0.5, 1.0, 2.0\}$. For CW-DFM, we also run the sweep at the test time to find the optimal neighbor solver parameters.

\subsection{Additional results}
\label{sec:add_results}

\paragraph{Entropy, NLL and Local context.} To illustrate the locality bias, we took a pre-trained DFM baseline model with masked source distribution trained on OpenWebText (OWT)~\citep{Gokaslan2019OpenWeb}. We sample $x_1$ from the validation subset and sample $x_t \sim p_t(x \mid x_0, x_1)$ at various timesteps for a total of 16k samples. For each masked token $i$ in each $x_t$, we measured the entropy of the model prediction $H\left(p_{1 \mid t}^{\theta, i}(\cdot \mid x_t)\right)$ and the negative log-likelihood of the correct token, $\mathcal{L}_i = -\log p_{1 \mid t}^{\theta, i}(x_1^i \mid x_t)$, a.k.a. the error of the prediction. We then evaluated two simple locality statistics: the number of unmasked tokens within a radius $r=8$, and the distance to the nearest unmasked token (see \cref{fig:entropy_nll_dist}). In both cases, stronger local context corresponds to lower entropy and lower NLL: increasing the number of visible neighbors decreases both metrics, while increasing the distance to the nearest visible token increases them. This suggests that several simple locality measures can serve as inexpensive proxies for prediction uncertainty.

\begin{figure}[ht]
    \centering
    \includegraphics[width=\linewidth]{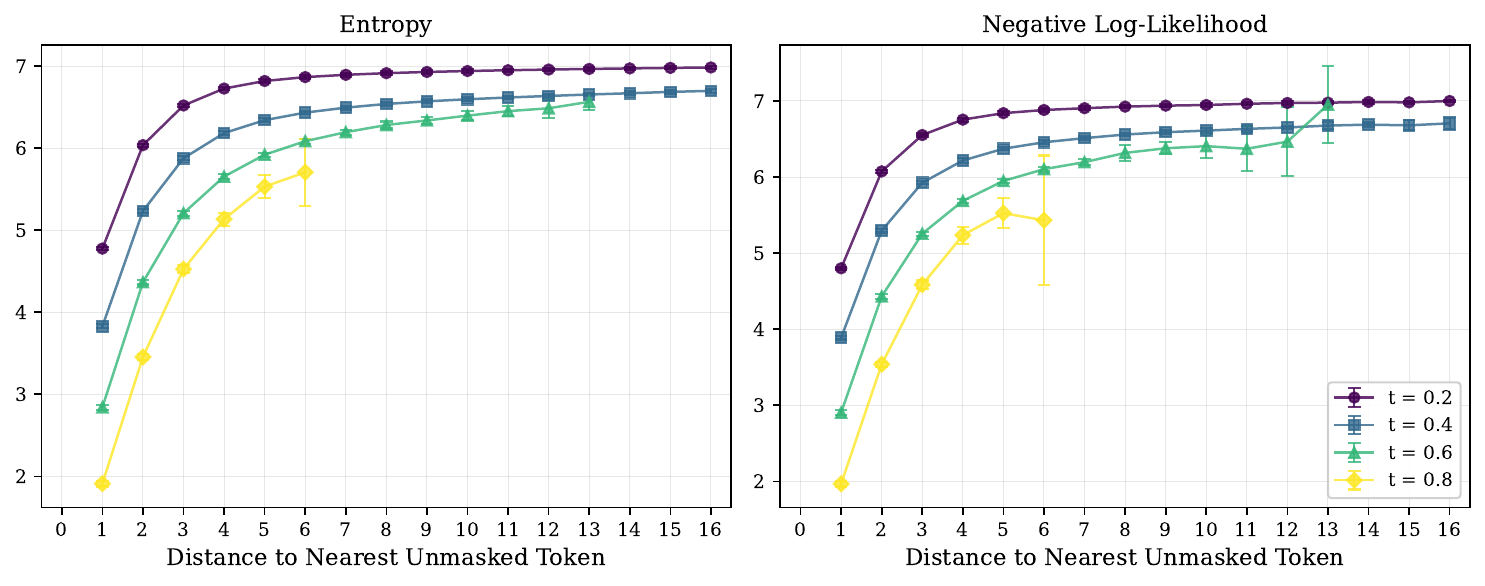}
    \caption{\textbf{The distance to the nearest unmasked token provides another proxy for the token-level uncertainty.} Both the entropy and negative  negative log likelihood increase with the distance, but saturate quite early. Vicinity radius $r=16$, $\kappa(t) = t^2$.}
    \label{fig:entropy_nll_dist}
\end{figure}

\begin{figure}[ht]
    \centering
    \includegraphics[width=\linewidth]{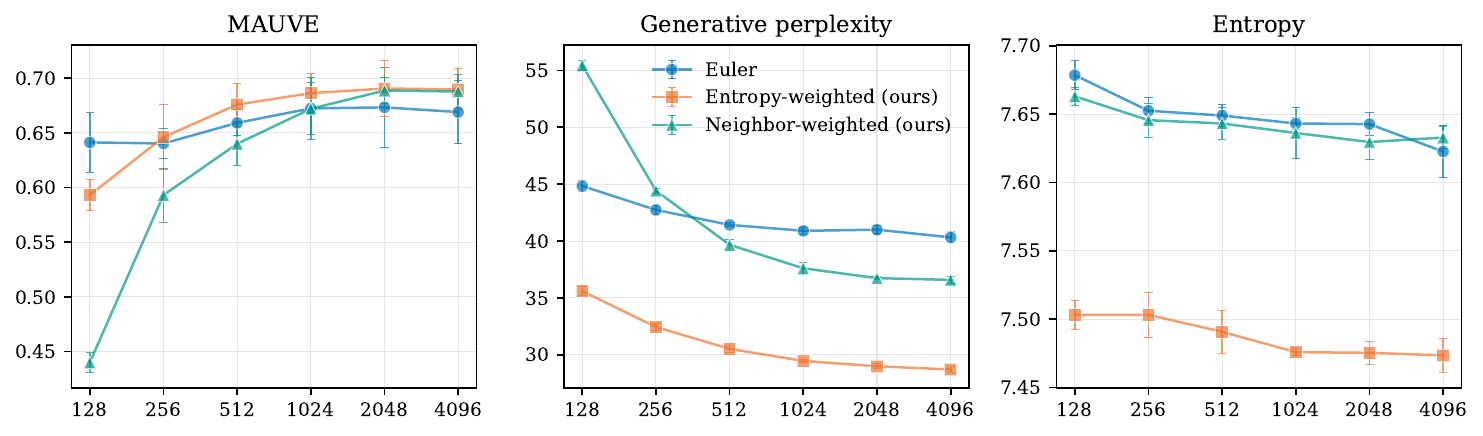}
    \caption{\textbf{Context-weighted sampling remains effective with masked source noise.} With a masked source distribution, context-weighted solvers improve generative perplexity over Euler sampling while maintaining comparable token-level entropy.}
    \label{fig:nfe_vs_quality_mask}
\end{figure}

\begin{figure}
    \centering
    \includegraphics[width=\linewidth]{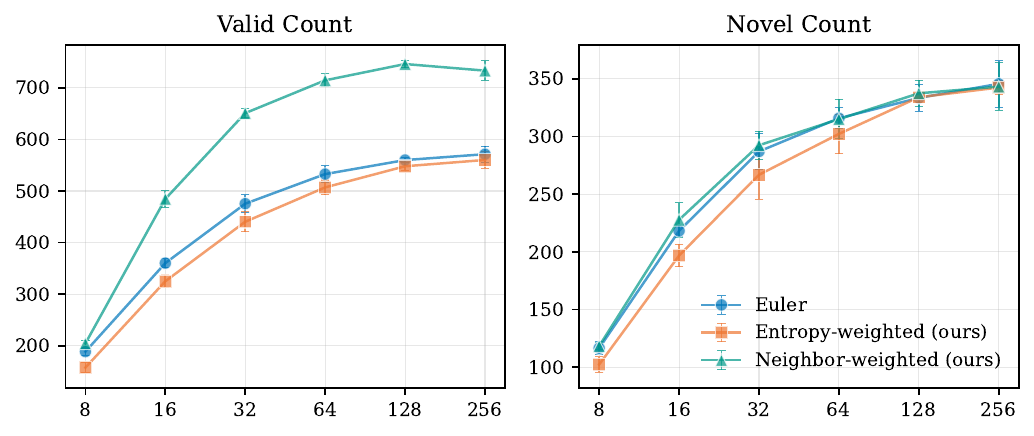}
    \caption{\textbf{Neighbor-weighted sampling improves molecular validity with uniform source noise.} On QM9, Neighbor-weighted sampling generates more valid molecules than the baselines across NFE, while maintaining a comparable number of novel molecules out of 1024 samples.}
    \label{fig:nfe_vs_quality_qm9_uniform}
\end{figure}

\begin{figure}[ht]
    \centering
    \includegraphics[width=\linewidth]{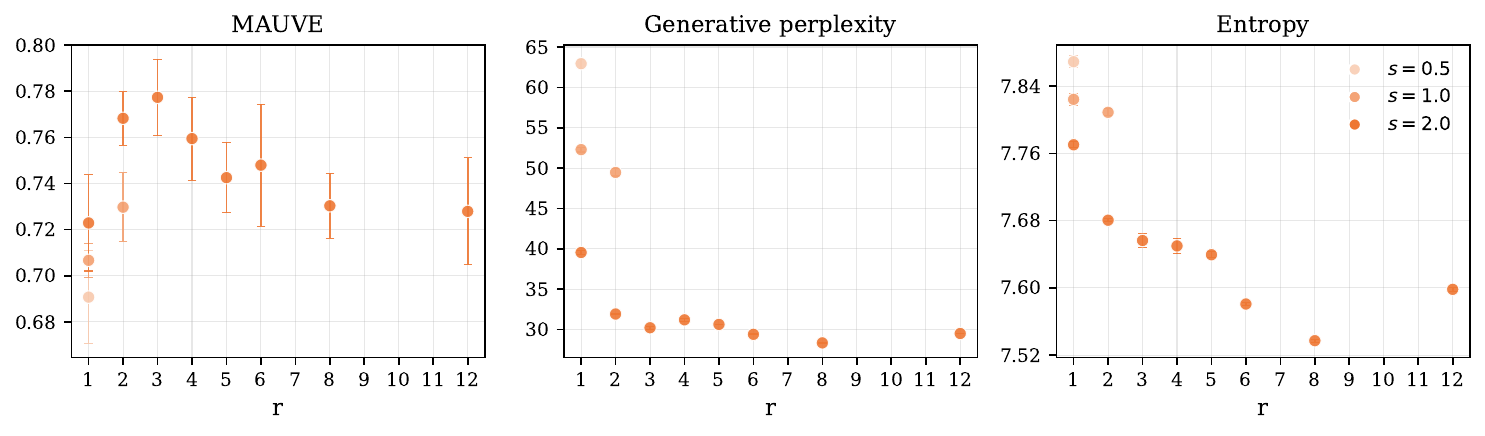}
    \caption{\textbf{Moderate local neighborhoods give the best quality--diversity trade-off.} For DFM trained with SCE on OpenWebText, increasing the neighborhood radius reduces the perplexity but also the entropy; MAUVE peaks at an intermediate radius, suggesting that too little or too much context aggregation is suboptimal.}
    \label{fig:vicinity_size_abl}
\end{figure}

\paragraph{Training dynamics.}
\cref{fig:training_dynamics} shows that SCE follows a smooth training trajectory and leads to substantially lower generative perplexity. The validation objectives are not directly comparable, since CE and SCE use different token weightings. The accompanying moderate entropy decrease suggests that SCE makes generations more coherent without collapsing diversity.


\begin{figure}[ht]
    \centering
    \includegraphics[width=\linewidth]{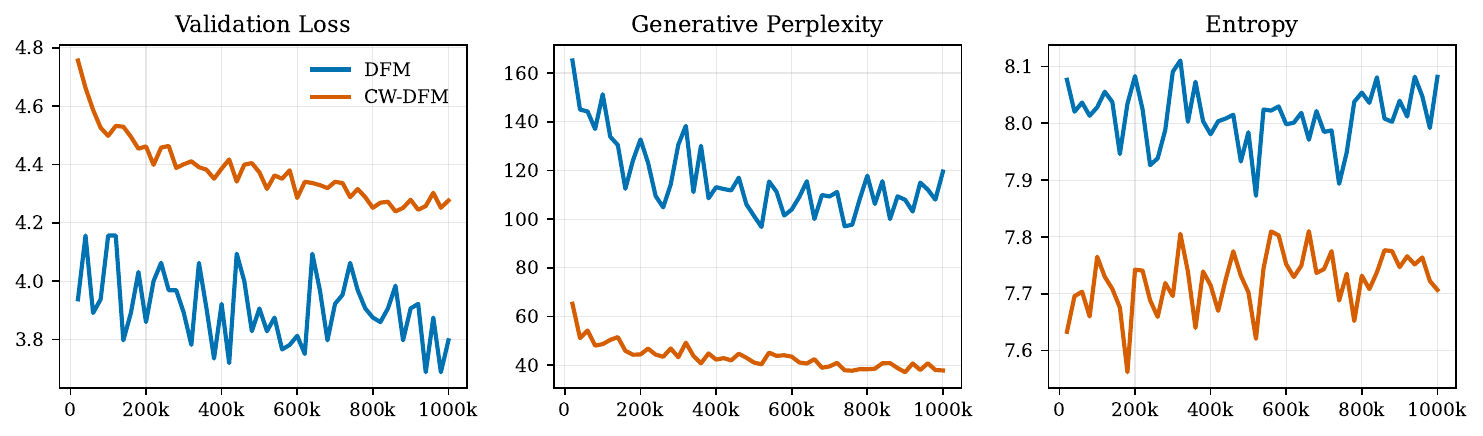}
    \caption{\textbf{SCE changes the optimization trajectory toward lower generative perplexity.}
    Although the weighted SCE validation loss is not directly comparable to standard CE, SCE decreases smoothly during training and yields generations with lower perplexity and a moderate reduction in entropy.}
    \label{fig:training_dynamics}
\end{figure}


\begin{table}[ht] 
    \caption{Test perplexity ($\downarrow$) for a variety of datasets.}
    \label{tab:elbo}
    \centering
    \begin{tabular}{ll|cccc}
        \toprule
        Model & Source & Wikitext103 & LAMBADA & FineWeb-Edu \\
        \midrule
        DFM & Uniform & $\leq$ 63.38 & $\leq$ 84.33 & $\leq$ 48.20 \\
        DFM+SCE & Uniform & $\leq$ 91.53 & $\leq$ 133.25 & $\leq$ 72.90 \\
        \midrule
        DFM & Mask & $\leq$ 45.33 & $\leq$ 57.88 & $\leq$ 35.00 \\
        DFM+SCE & Mask & $\leq$ 45.27 & $\leq$ 59.59 & $\leq$ 34.80 \\
        \bottomrule
    \end{tabular}
\end{table}

\paragraph{Test perplexity.} We report test perplexity on Wikitext103\citep{merity2016pointer}, LAMBADA\citep{paperno2016lambada}, and FineWeb-Edu \citep{lozhkov2024fineweb-edu} datasets in \cref{tab:elbo}. We observe a discrepancy between test perplexity and generation quality: models trained with SCE can have higher test perplexity under the standard evaluation, while producing better samples according to generative perplexity, MAUVE, and qualitative inspection. This is expected because SCE deliberately changes the weighting of the prediction problem rather than optimizing the unweighted likelihood of all coordinates equally. Standard test perplexity evaluates all token predictions uniformly, including highly ambiguous positions with little available context, whereas SCE prioritizes well-contextualized predictions that provide a sharper learning signal. As a result, SCE may be less favorable under the standard perplexity objective, but better aligned with the iterative sampling process, where improving reliable local updates can compound into more coherent generations. A fully aligned likelihood-style evaluation would require accounting for the same context-dependent weighting or path used during training.

\begin{figure}[ht]
    \centering
    \includegraphics[width=0.8\linewidth]{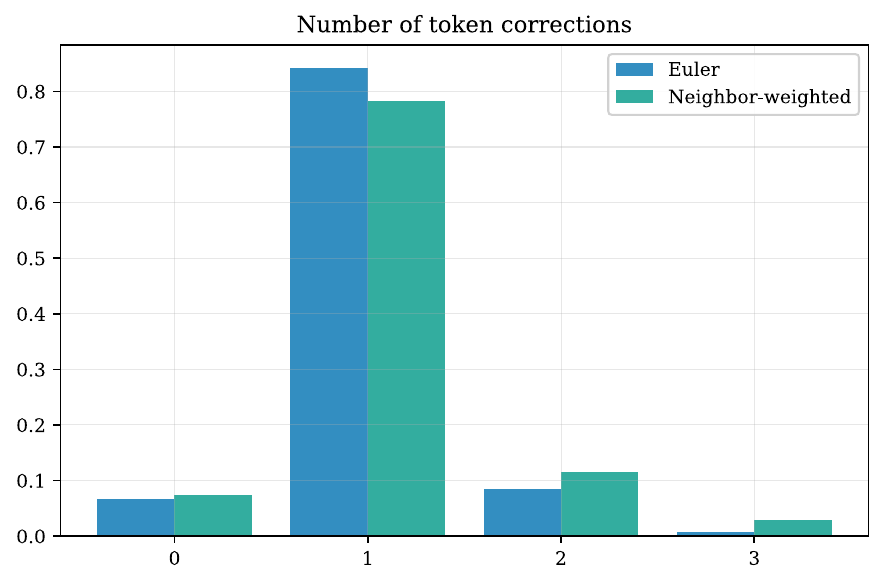}
    \caption{Distribution of the number of per-token updates for Euler and Neighbor-weighted solver, uniform source. The Neighbor-weighted solver has higher number of per-token corrections on average.}
    \label{fig:token_corrections}
\end{figure}

\paragraph{Self-correction.} Naturally, we observe that Neighbor-weighted solver performs more updates per token than the standard Euler sampler (see \cref{fig:token_corrections}). For Euler solver, the average number of per-token corrections was $1.033\std{0.001}$, whereas for Neighbor-weighted solver, we have $1.157\std{0.001}$ corrections (uniform source). There is very little redundancy in these updates, since the proportion of loops for Neighbor-weighted sampling with uniform source is 0.02\%.

\paragraph{Scale ablation.} \cref{tab:abl_cwdfm} ablates the inverse-temperature parameter $s$ used in the context-weighted path $\tilde{p}_t$. Increasing $s$ makes the weighting more selective, which generally lowers generative perplexity but can also reduce entropy. The best value depends on the source distribution: for the uniform source, $s=0.5$ gives slightly higher MAUVE and diversity, while $s=1.0$ yields much lower perplexity; for the masked source, $s=1.0$ provides the best overall trade-off.

\paragraph{Societal impact.}
This work is methodological and aims to improve the efficiency and controllability of discrete generative models. Potential positive impacts include reducing sampling cost and improving the accessibility of generative modeling research. At the same time, stronger generative models can be misused for synthetic text generation at scale, including spam or disinformation. Our work does not introduce a deployed system or use sensitive data, but future applications should consider appropriate monitoring and safeguards.

\begin{table}[ht]
    \caption{Ablation of the scale parameter $s$ for CW-DFM ($\tilde{p}_t$).}
    \label{tab:abl_cwdfm}
    \centering
    \begin{tabular}{ll|ccc}
    \toprule
    Source & $s$ & MAUVE & Gen. PPL & Entropy \\
    \midrule
    \multirow{2}{*}{Uniform} & $0.5$ & 0.779\std{0.017} & 48.91\std{0.29} & 7.83\std{0.01}\\
    & $1.0$ & 0.768\std{0.017} & 37.45\std{0.38} & 7.70\std{0.01} \\
    \midrule
    \multirow{2}{*}{Mask} & $0.5$ & 0.712\std{0.023} & 35.48\std{0.26}& 7.60\std{0.01} \\
    & $1.0$ & 0.751\std{0.012} & 36.25\std{0.52}& 7.65\std{0.01} \\
    \bottomrule
    \end{tabular}
\end{table}

\end{document}